 \documentclass[final,5p,times,twocolumn]{elsarticle}

\usepackage[numbers]{natbib}
\usepackage{graphicx}
\usepackage{float}
\usepackage{subfigure}
\usepackage{multirow}
\usepackage{amsbsy}
\usepackage{amssymb}
\usepackage{amsmath}
\usepackage[utf8]{inputenc} 
\usepackage[T1]{fontenc}    
\usepackage{hyperref}       
\usepackage{url}            
\usepackage{booktabs}       
\usepackage{amsfonts}       
\usepackage{nicefrac}       
\usepackage{microtype}      
\usepackage{xcolor}         
\usepackage{algorithm}
\title{Grid-Mapping Pseudo-Count Constraint for Offline Reinforcement Learning}
\usepackage{algorithmic}
\usepackage{appendix}
\usepackage{amsthm}
\usepackage{subfigure}
\usepackage{subcaption}
\usepackage{natbib}
\usepackage{multirow}
\usepackage{setspace}
\setcitestyle{numbers,square}

\newtheorem{theorem}{Theorem}
\newtheorem{theorem A.}{Theorem A.}
\newtheorem{definition}{Definition}
\newtheorem{corollary}{Corollary}
\newtheorem{lemma}{Lemma}

\newtheorem{assumption A.}{Assumption A.}
\newtheorem{lemma A.}{Lemma A.}

\journal{Knowledge-Based Systems}

\begin{document}

\begin{frontmatter}


\author[university]{Yi Shen}
\ead{sheny89@mail2.sysu.edu.cn}
\author[university]{Hanyan Huang \corref{corresponding}}
\ead{huanghy99@mail.sysu.edu.cn}
\cortext[corresponding]{Corresponding author.}


\affiliation[university]{organization={Sun Yat-Sen University},
            addressline={135 Xingang West Road, Guangzhou}, 
            city={Guangzhou},
            postcode={510275}, 
            state={Guandong},
            country={China}}


\begin{abstract}
Offline reinforcement learning learns from a static dataset without interacting with environments, which ensures security and thus owns a good application prospect. However, directly applying naive reinforcement learning algorithm usually fails in an offline environment due to inaccurate Q value approximation caused by out-of-distribution (OOD) state-actions. It is an effective way to solve this problem by penalizing the Q-value of OOD state-actions. Among the methods of punishing OOD state-actions, count-based methods have achieved good results in discrete domains in a simple form. Inspired by it, a novel pseudo-count method for continuous domains called Grid-Mapping Pseudo-Count method (GPC) is proposed by extending the count-based method from discrete to continuous domains. Firstly, the continuous state and action space are mapped to discrete space using Grid-Mapping, then the Q-values of OOD state-actions are constrained through pseudo-count. Secondly, the theoretical proof is given to show that GPC can obtain appropriate uncertainty constraints under fewer assumptions than other pseudo-count methods. Thirdly, GPC is combined with Soft Actor-Critic algorithm (SAC) to get a new algorithm called GPC-SAC. Lastly, experiments on D4RL datasets are given to show that GPC-SAC has better performance and less computational cost than other algorithms that constrain the Q-value.
\end{abstract}



\begin{keyword}
reinforcement learning \sep Offline reinforcement learning \sep count-based \sep OOD prediction



\end{keyword}

\end{frontmatter}



\section{Introduction}

These years, with the release of chatgpt4 and swift~\cite{openai2023gpt4,kaufmann2023champion}, the achievement of deep reinforcement learning (DRL) has once again been pushed to a new peak. In addition to chatbots and automatic operation, there are also many remarkable achievements have been made by DRL in many other fields, such as games~\cite{silver2018general,vinyals2019grandmaster,berner2019dota} and medical~\cite{yu2021reinforcement}. It is worth noting that most of these achievements do not require the agent to interact with the physical environment. Why are the current achievements of RL in physical environments less than the achievements in virtual environments? That is because DRL is learning what to do—how to map situations to actions—to maximize a numerical reward signal~\cite{sutton2018reinforcement}. In other words, the application of RL in a physical environment requires the interaction between agents and the environment, which may cause great damage to the environment or the agent~\cite{levine2020offline}. To avoid this problem, offline reinforcement learning(offline RL) that does not require interaction between agents and the environment has become a hot research topic. Offline RL, also known as batch RL, only learns policy through a static dataset without interacting with the environment, thus guaranteeing safety. Thus offline RL is considered to have the potential to be large-scale used in physical scenarios~\cite{fujimoto2019off,agarwal2020optimistic}. 

However, there is a problem called distributional shifts when RL algorithms are directly used in offline scenarios. It is mainly attributed to the existence of actions or states that are not covered by the static dataset during training. This type of state and action is called out-of-distribution (OOD) state and action. 
In offline RL, function approximators are usually used to approximate the state-action value Q or the state value V. Although function approximators can approximate the Q-value or V-value of data in the static dataset accurately, for the states or actions that do not exist in the static dataset, function approximators may not be able to obtain accurate estimates.~\cite{fujimoto2019off}. This may cause the agent to choose some suboptimal actions with overestimated Q-values or V-values, ultimately resulting in a poor policy through training. 

Common-used methods to deal with OOD state-actions can be divided into policy constraint, uncertainty estimation, model-based, regularization, and importance sampling~\cite{prudencio2023survey,fujimoto2021minimalist,schulman2015trust,an2021uncertainty}. 

Algorithms using the above methods can effectively solve the OOD problem. However, most current algorithms only treat static datasets as environments and do not consider further exploitation of other information already available in static datasets. For example, in Q-value constrained algorithms, the most popular method currently is to use the ensemble method to generate multiple Q-value approximators to estimate the Q-value and obtain a conservative Q-value~\cite{kurutach2018model}. This type of method does not further utilize the data in the static dataset, which greatly increases the computation time and space required for training. We believe that further exploration of information in static datasets can further improve the performance of algorithms. To test this assumption, we consider the count-based method. The count-based method is another method to constrain the Q-value. Although the count-based method is not very effective in general RL due to its low exploration efficiency, it has the characteristic of efficiently utilizing prior information~\cite{osband2018randomized}. Therefore, we believe that the count-based method has the potential to improve algorithm performance in offline RL by efficiently utilizing data from static datasets.

In discrete space, the count-based method can directly count the number of state-action pairs selected due to the limited number of state actions. But in continuous space, intuitively counting becomes impossible due to the infinite possible state-action pairs. Thus, in continuous offline scenarios, automatic encoders are generally used to process state-action pairs and count the processed state-action pairs. In continuous offline scenarios, the binary state-action pairs processed by an auto-encoder are counted to constrain the Q-value. Although good results can achieved by using an auto-encoder, it is unable to theoretically justify the uncertainty constraints obtained, and the computational cost may increase due to the auto-encoder. To address these challenges, we propose the Grid-Mapping Pseudo-Count method(GPC) and in this article, our contributions are as follows: We first propose the GPC method, GPC maps continuous state and action spaces to grid spaces using information from the static dataset and uses the pseudo-count to constrain the Q-value of state-action pairs in the grid space. The learned policy is used by GPC to collect OOD state-action samples and OOD state-actions are constrained by the uncertainty quantified by GPC. Secondly, We theoretically prove that GPC can approximate the true uncertainty with fewer assumptions compared to other algorithms that use the count-based method. Thirdly, we propose the Grid-Mapping Pseudo-Count Soft Actor-Critic algorithm(GPC-SAC) by combining GPC with the Soft Actor-Critic(SAC) algorithm. Finally, we experimentally demonstrated the effectiveness of GPC-SAC in D4RL~\cite{fu2020d4rl}.

\section{Related Work}

In offline RL algorithms using the count-based method, the idea of uncertainty estimation is mainly used to estimate uncertainty~\cite{gal2016improving,kidambi2020morel,lee2021sunrise}, and then the quantified uncertainty is used to constrain the Q-value ~\cite{agarwal2020optimistic,prudencio2023survey,mao2024supported,ma2024mutual}. 

The count-based method was first used in combination with the model-based method in RL~\cite{hoel2020reinforcement,clements2019estimating}. After that, count-based methods have been extensively studied in off-policy RL. For count-based methods, the novelty of a state is measured by the number of visits, then a bonus is assigned accordingly. Count-based exploration bonus~\cite{bellemare2016unifying} refers to UCB, a traditional RL method for exploration to DRL, where the counting of states is regarded as an indicator of uncertainty and then is used as an intrinsic motivator to encourage exploration. In DQN-PixelCNN~\cite{ostrovski2017count} PixelCNN~\cite{van2016conditional} is used to measure the probability density of the current state, reconstruct the form of intrinsic motivation, and multi-step updates are used to achieve better results. Domain-dependent learned hash code~\cite{tang2017exploration} further simplifies the count-based method and extends it to continuous RL by using an autoencoder to represent the hash function, mapping the state to a low dimensional feature space, and counting in the feature space. The current count-based methods in RL mainly refer to the structure in this article. In subsequent related research, in A2C+CoEX~\cite{choi2018contingency}, action-related regions are extracted from original image observations and the count of each action-state pair are used to calculate intrinsic motivation. One inherent problem of count-based methods is that the next step is explored based on the current state-actions in count-based methods. In off-policy RL, there may be significant differences between estimated state-actions and actual state-actions, leading to agents being unable to effectively explore~\cite{osband2018randomized}. 

 In offline RL, agents are required to learn and further explore information from static datasets, and the effect of count-based method is better when the Q-value estimation of it is better. Therefore, we believe that the performance of count-based methods can be better in offline RL. There are currently some research results on the use of count-based in offline RL. TD3-CVAE~\cite{rezaeifar2022offline} instantiated with a bonus based on the prediction error of a variational autoencoder and studied the problem of continuous fields. However, it could not reach state-of-the-art at that time. Using count-based methods in discrete fields, CCVL~\cite{hong2022confidence} achieves state-of-the-art in discrete space, but it cannot be applied in continuous fields. Combines TD3-CVAE with model-based methods, count-MORL~\cite{kim2023model} achieving SOTA in the model-based field. However, the use of networks in both state-action modeling and counting increases the complexity of the algorithm. Hence, the research on count-based methods in offline RL is valuable.

\section{Preliminaries}

\textbf{Reinforcement learning}: In RL, the environment is generally represented by an episodic Markov Decision Process(MDP) defined as $(S,A,r,P,{\rho _{0}},\gamma )$~\cite{fujimoto2019off}. Here, $S$ is the state space, $A$ is the action space, $r:S \times A \to [R_{\min },{R_{\max }}]$ is the reward function, $P:S \times A \to S'$ is the transition dynamics. By $P$, the agent takes action to transition to a new state in the current state, $\rho _0 \in S$ is the initial state distribution and $\gamma \in (0,1]$ is the discount factor, which is used to adjust the impact of state-actions over time. 

In this context, RL discusses how an agent finds an optimal policy $\pi:S \to A$, where the objective is to maximize the expected return over a finite time horizon $T$, given by $\max{E_{a_{i}\sim \pi(a_{i}|s_{i})}[\sum_{i=0}^{T-1} \gamma^{i} r(s_{i},a_{i})]}$. Here, ${a_{i}\sim \pi(a_{i}|s_{i})}$ represents that $a_{i}$ follows the distribution $\pi(a_{i}|s_{i})$. Q-learning is one of the main ways to find the optimal policy~\cite{haarnoja2018soft}. Q-learning iteratively selects actions based on the current state to maximize the obtained rewards until the optimal policy is achieved. The Q-function corresponding to the optimal policy is learned by satisfying the following Bellman operator $B$:
\begin{equation}
\begin{array}{l}
    BQ_{\phi}(s_i,a_i) = r(s_i,a_i)+ \\ \gamma E_{s_{i+1}\sim P(s_{i+1}|s_i,a_i),a_{i+1}\sim \pi(a_{i+1}|s_{i+1})}[Q_{\phi^-} (s_{i+1},a_{i+1})]
\end{array}
\end{equation}
Where $\phi$ is the parameter of the Q-network. The Q-function is updated by minimizing the TD-error $E_{(s,a,r,s')}[(Q_\phi-BQ_\phi)^2]$. The target-Q $BQ_{\phi}(s_i,a_i)$is usually computed by a separate target network which is parameterized by ${\phi ^ - }$~\cite{mnih2015human}. 

\textbf{Offline Reinforcement learning}: In offline RL, a policy is no longer learned by interacting with the environment, but a policy instead tries to be learned from a static dataset $D = \{ ({s_i},{a_i},{r_i},s{'_i})\} _{i = 1}^n$, where the data can be heterogenous and suboptimal. 
In offline RL, the distribution shift phenomenon will occur when applying the naive RL algorithm, due to the different distribution of behavior policy and learning policy. Consequently, the Q-value of OOD state-action pairs will be overestimated and then the suboptimal ones will be wrongly selected. Moreover, this mistake cannot be corrected by interaction with the environment in offline RL, thus the error will be propagated through the Bellman operator and a poor training result will be leaded~\cite{kumar2019stabilizing}. 

\textbf{Count-based method}: The count-based method generally constrains the state-action value by counting the number of times a state-action is selected, in discrete space, the most basic constraint method is $B{Q_\phi }(s^{ood},a^{ood}) = {Q_\phi }(s^{ood},a^{ood}) - \frac{1}{{\sqrt {n(s^{ood},a^{ood})} }}$. Here $n(s^{ood},a^{ood})$ is the number of times an OOD state-action has been selected. This decreasing constraint can be understood as the more times a state action pair is selected, the more accurate the estimation of the Q-value can be. Therefore, after a sufficient number of visits, it is no longer necessary to restrict the Q-value too much. and such a simple constraint can achieve good results. In continuous space, $n'(s^{ood},a^{ood})$ are generally used in place of $n(s^{ood},a^{ood})$ can also give good results, here $n'$ is the pseudo-count of state-actions $(s^{ood},a^{ood})$ .

 \section{Grid-Mapping Uncertainty For Offline RL}
 
In this section, we propose a new pseudo-counting method called the Grid-Mapping method and we combine it with pseudo-counting to obtain a count-based method called the Grid-Mapping Pseudo-Count method(GPC). We then show from a theoretical point of view that GPC can correctly constrain OOD state-action pairs.
Finally, we propose the GPC-SAC algorithm by combining GPC with SAC\cite{haarnoja2018soft} framework.

\subsection{Grid-Mapping Pseudo-Count Method}
\label{4.1}
Since it is not possible to count all state-actions directly in a continuous environment, we use the prior information of states and actions in the static dataset to map the state-action space and count the mapped state-actions. In detail, in the Grid-Mapping method, each dimension of the state-action space is gridded respectively, we map state-actions to a discrete space by the maximum and minimum values of states and actions in each dimension. The operation of gridding the $i$-th dimensional of the action or the $j$-th dimensional of the state is formulated as follows: 
\begin{equation}
\begin{array}{l}
s'_{i} = ({\alpha}({s_i} - s_i^{\min }))\bmod (m(s_i^{\max } - s_i^{\min }))\\
a'_{j} = ({\alpha}({a_j} - a_j^{\min }))\bmod (m(a_j^{\max } - a_j^{\min }))
\end{array}
\label{gridOpertor}
\end{equation}
Where $ s_i^{min}$ and $s_i^{max}$ are respectively the minimum and maximum values of the $i$-th dimension of states in the static dataset; $a_{j}^{min}$ and $ a_{j}^{max}$ are respectively the minimum and maximum values of the $j$-th dimension of actions in the static dataset. $\alpha$ is the number of partitions of the state space and action space in the static dataset. In addition, to prevent misclassifying some of the OOD state-actions as non-OOD state-actions, $m \ge 1$ can be used to adjust the influence of OOD state-actions. By Eq\eqref{gridOpertor}, any state-action $(s, a)$ can converted to another corresponding state-action $(s', a')$. For convenience, $(s',a')$ is then converted into an integer $v(s',a')$ as follows:
\begin{equation}
\label{3}
\begin{aligned}
   v(s',a')=& \sum^{n_{1}-1}_{i=0} {s'_{i} \times (\alpha^{i}}+1) + \sum^{n_{2}-1}_{j=0} a'_{j} \times (\alpha^{(n_{1}-1)+j}+1)
\end{aligned}
\end{equation}
Where $n_{1}$ and $n_{2}$ are respectively the dimensions of state space and action space. In this way, counting $v(s',a')$ is equivalent to counting the state-action pair $(s,a)$ corresponding to $(s',a')$. The entire process of obtaining pseudo-counting through GPC can be seen in Figure \ref{figure01}.

\begin{figure*}
    \centering
    \includegraphics[width=0.99\linewidth]{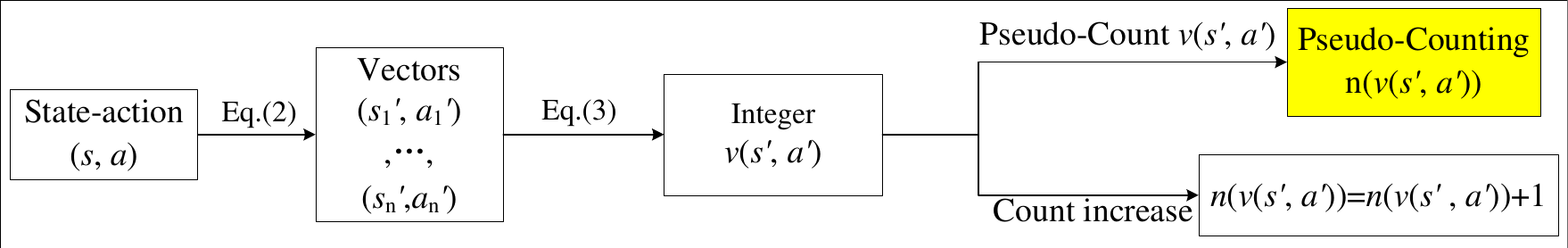}
    \caption{obtaining pseudo-counting through GPC}
    \label{figure01}
\end{figure*}  
Since the original data distribution is retained by GPC, the distributional features of uncertainty in the state-action space can be preserved by GPC and the uncertainty of different state-actions can be accurately quantified by GPC. Moreover, due to this feature, the ability of GPC to obtain an approximation of the true state-action uncertainty can be theoretically justified in the case of continuous state-action spaces. 
Notably, the states used in the training process of most current algorithms are all existing states in the static dataset, thus it is feasible to directly count all possible states in the static dataset instead of gridding the state. 

\subsection{Agent Learning With Uncertainty} 
\label{4.3}
After the pseudo-count, $n'(v(s,a))$ was obtained, we show how $n'(v(s,a))$ can be used to constrain the Q-value of the out-of-distribution state pairs.

First, $n'(v(s,a))$ is used to calculate $u(s,a)$ through $u(s,a) = \alpha \sqrt {\frac{{\ln T}}{{n'(v(s,a))}}} $. Where $\alpha$ is the hyperparameter used to scale the uncertainties, $T$ is the training episode. Then, $u(s,a)$ is used to update the Q-function by constraint OOD samples. OOD samples are sate-actions used in the training process but are not covered by the static dataset. Then, inaccurate Q-value estimations may given by the Q-function approximator for these samples. Sub-optimal samples may be utilized by agents due to inaccurate estimations and this type of use of sub-optimal samples may lead to poor results for the obtained policy. Due to this, the quantified uncertainty is used to constrain them. Since we consider the Q-values for OOD state-actions and in-distribution state-actions separately, the Q-value loss function is divided into two parts: in-distribution loss and OOD loss. The in-distribution loss is the conventional TD-error determined by the Bellman operator, which has the same form as other algorithms can be obtained as follows:
\begin{equation}
   B^{ in} Q^{ k}_{\phi} (s_{ in},a_{ in}) =  r(s_{ in},a_{ in}) + \gamma \hat{E}_{a_{ next} \sim \pi (.|s)}[Q^{ k}_{\phi^-} (s^{ next}_{ in},a_{ next})] 
\end{equation}
\begin{equation}
     {L_{ in}} = {{\hat E}_{({s_{ in}},{a_{ in}},r,s_{ in}^{ next})\sim {D_{ in}}}}[{(({B^{ in}}Q_\phi ^k({{\mathop{ s}\nolimits} _{\rm in}},a_{ in}) - Q_{ \phi} ^{ k}({s_{ in}},a_{ in}))^2}]
\label{inloss}
\end{equation}
Where $Q^{ k}_{\phi^-}$ is the $k$-th target Q-network, $Q^{k}_{\phi}$ is the corresponding result of the $k$-th Q-network, $\phi^-$, $\phi$ denote the parameters. $(s_{ in},a_{ in})$ is state-action already in the static dataset, $(s^{ next}_{ in},a^{ next}$) is the next state after $(s_{ in},a_{ in})$ and the action $a_{ next}$ satisfies $a_{ next} \sim \pi$.

One of the problems in GPC is how to get OOD state-action pairs and constraint them, here we use the actions $a^{ ood}_{ in}$, $a^{ ood}_{ next}$ following policy under states $s_{ in}$ and $s_{ next}$, and constrain $(s_{ in}, a^{ ood}_{ in})$ and $(s_{ next}, a^{ ood}_{ next})$ as OOD state actions.
Similar to $B^{ in} Q^{ k}_{\phi} (s_{ in},a_{ in})$, the target Q-values of OOD samples are formulated as follows. 
\begin{equation}
    B^{ood}Q^{k}_{\phi}(s_{in}, a^{ood}_{in}) =  Q^{k}_{\phi}(s_{in}, a^{ood}_{in}) -\beta u(s_{in}, a^{ood}_{in})
\end{equation}
\begin{equation}
    B^{ood}Q^{k}_{\phi}(s_{next}, a^{ood}_{next})  =  Q^{k}_{\phi}(s_{next}, a^{ ood}_{next}) -\beta_{next}  u(s_{next}, a^{ood}_{next})
\label{oodloss}
\end{equation}
In this equation, $\beta$ and $\beta_{next}$ are the hyperparameters adjusting the degree of constraint. Here ${(s_{in} , a^{ood}_{in})}$ and ${(s_{in} , a_{in})}$ have the same state, but $a_{in}$ and $a^{ood}_{in}$ may be different. 

And for $L_{ood}$, we directly add the above two parts linearly to get $L_{ood}$:
\begin{equation}
    \begin{array}{c}
L_{ood} = {{\hat E}_{s\sim {D_{ in}},{a_{ in}^{ ood}}\sim \pi }}[{({B^{ ood}}Q_\phi ^k({s_{ in}},a_{ in}^{ ood}) - Q_\phi ^k({s_{ in}},a_{in}^{ood}))^2}\\
 + {({B^{ ood}}Q_\phi ^k({s_{ next}},a_{ next}^{ood}) - Q_\phi ^k({s_{ next}},a_{ next}^{ ood}))^2}]
\end{array}
\end{equation}

Finally, the loss function $L_{ total}$ is expressed as the sum of the OOD loss and the in-distribution loss as follows:
\begin{equation}
    L_{total} = L_{in}+L_{ ood}
\end{equation}

In order to better utilize the constrained OOD state-action pairs to assist exploration, we update the policy using the following approach~\cite{haarnoja2018soft}:
\begin{equation}
\label{policyupdt}
{\min_{k}} {\nabla _{\theta} } Q_{\phi}^{k}({s_{ in}},{a^{ ood}_{ in}}) - \psi  \log{\pi_{\theta}}(a^{ ood}_{ in}|{s_{ in}})
\end{equation}
Here $\psi$ is a hyperparameter, $\theta$ is the parameter of policy $\pi$. Exploration as the way above makes use of OOD state-action pairs therefore able to take advantage of the ability of GPC to impose reasonable constraints on OOD state-actions.

The proposed algorithm is trained by updating the Q-function and the policy by the above methods.
In the early stage of training, large uncertainty makes the estimate of the Q-value of the OOD sample pessimistic, thus the existing state-action pairs are more inclined to be selected by the agent. In the middle and later stages of training, as the uncertainty decreases, the samples that are outside the dataset but may perform well are likely to be chosen by the policy, leading to a balance between conservatism and exploration.

\subsection{Connection Between Pseudo-Count And Uncertainty }
\label{theory_proofs}
For continuous space, it is impossible to obtain the uncertainty constraint directly by counting the selection times of a certain state-action. Therefore, the pseudo-count $n'(s,a)$ is generally used to approximate a real count $n(s,a)$ to obtain the uncertainty constraint. However, it is difficult to prove whether the pseudo-counting can accurately respond to the degree of uncertainty about a state-action. But GPC proposed in this paper can guarantee that the pseudo-counts obtained by it can correctly represent the uncertainty of a state-action under a small number of assumptions. Here is a brief proof of the conclusion.

\begin{lemma} 
\label{co1}
When selected appropriate hyperparameter $\alpha$, $u(s,a) = \alpha \sqrt {\frac{{\ln{T}}}{{n(s,a)}}}$ is a suitable uncertain constraint in discrete offline RL.
\end{lemma}

Lemma \ref{co1} can obtained by Hoeffding's inequality directly~\cite{hoeffding1994probability}. If you are interested, the proof of the lemma is in our appendix.

From the above Lemma it can be inferred that in discrete space, appropriate uncertainty constraints can be obtained by GPC. Then, the results of discrete spaces can be extended to continuous spaces. 
For continuous spaces, we consider epistemic uncertainty. So we prove that an approximation of the epistemic uncertainty can be obtained by GPC. Firstly, we get an approximation of the epistemic uncertainty constraint in another way. Secondly, we prove that the uncertainty obtained by GPC is equivalent to the epistemic uncertainty constraint. 
Referring to the work of pioneers, an approximation of the uncertainty constraint can be obtained by the following process~\cite{wang2020reward}.
\begin{definition}
\label{linearmdp}
An MDP is linear if it satisfies that the state transition function and reward function are linear to the transition kernel function $\varphi :S \times A \to {R^d}$.
\end{definition}

Although the linear hypothesis seems very strict, actually as the expression of the transition kernel of state and action is unlimited, the assumption of linear MDP has a powerful expression ability.

\begin{definition}
\label{uncer}
${\it u}$ is a suitable epistemic uncertainty function if it satisfies the following equation:
\begin{equation}
\label{uncertaintyfunc}
\begin{array}{c}
\forall (s,a) \in S \times A\\
P\left( {\left| {\hat B{V_{t + 1}}(s,a) - B{V_{t + 1}}(s,a)} \right| \le u(s,a)} \right) \ge 1 - \xi 
\end{array}    
\end{equation}
Where ${\it B}$ is the Bellman equation, $\hat B$ is the empirical Bellman equation estimated through the static datasets~\cite{wang2020reward}, $\xi$ can be any number in $(0,1]$.
\end{definition}

Definition \ref{uncer} is just a common definition for determining whether uncertainty is appropriate. From Definition \ref{linearmdp} and Definition \ref{uncer}, a quantification of epistemic uncertainty can obtained.
\begin{lemma}
\label{lemma2}
In linear MDPs, the reward function and transition kernel assumptions are linear to the representation of state-actions $\varphi :S \times A \to {R^d}$. When selecting appropriate hyperparameter $b_t$, the following lower confidence limit penalty is an appropriate uncertainty quantification~\cite{jin2020provably,abbasi2011improved}:
\begin{equation}
{\Gamma ^{\rm lcb}}(s_{ t},a_{ t}) = {b _ t} \cdot {[\varphi {({s_ t},{a_ t})^T}\Lambda _t^{ - 1}\varphi ({s_ t},{a_  t})]^{\frac{1}{2}}}
\label{lcb}
\end{equation}
where ${\Lambda _t} = \sum\nolimits_{i = 1}^m {\varphi (s_t^i,a_t^i)} \varphi {(s_t^i,a_t^i)^T} + \lambda  \cdot I$, $m$ is the sum of all state-action pairs in the dataset.
\end{lemma}

By substituting Eq\eqref{lcb} into Eq\eqref{uncertaintyfunc} to get:
\begin{equation}
P\left( {\left| {\hat B{V_{t + 1}}(s,a) - B{V_{t + 1}}(s,a)} \right| \le {\Gamma ^{\rm lcb}}(s,a)} \right) \ge 1 - \xi 
\label{eq13}
\end{equation}
Under the above conditions, the following corollary can be given.

\begin{corollary}
\label{cont}
In linear MDPs, ${\Gamma ^{\rm lcb}}(s,a)$ can be regarded as a continuous function. That is, the uncertainty in continuous offline RL can be approximated by a continuous function.
\end{corollary}
Although it is difficult to prove that ${\Gamma ^{\rm lcb}}(s,a)$ is a continuous function for all $\varphi$, it holds true for all the commonly used kernel functions. A detailed description is given in our appendices.

Based on the above conditions, Theorem \ref{the1} is inferred as follows:
\begin{theorem}
\label{the1}
In a continuous linear MDP, when the state space and action space are not unlimited. For the uncertainty constraint ${\it u'}$ obtained by GPC, there exists a parameter ${c}$ such that for any $(s,a)$ has:
\[P\left( {|\hat B{V_{t + 1}}(s,a) - B{V_{t + 1}}(s,a)| \le c u'(s,a)} \right) \ge 1 - \xi \]
\end{theorem}

\begin{proof}
When the state-action space is gridded using the Grid-Mapping method, the number of times that all state-action pairs in the region $[{s_{i - 1}},{s_i}) \times [{a_{j - 1}},{a_j})$ are selected is exploited to approximate the uncertainty of any point $(s,a)$ in it, using $u'$ to represent the uncertainty obtained by this method. From Corollary \ref{co1}, it can be obtained that:
\begin{equation}
\label{eq14}
\forall s,a \in [{s_{i - 1}},{s_i}) \times [{a_{j - 1}},{a_j}) \quad  u'(s,a) = \sqrt {\frac{lnT}{{n'(s,a)}}} 
\end{equation}
where \[n'(s,a) = \int_{s_{i - 1}}^{s_{j}} {\int_{a_{j - 1}}^{a_{j}} {n(s,a)dsda} } \] and because $n(s,a) \ge 0$ at any $(s,a)$, then from the Mean value theorems for definite integrals, it can be obtained that :
\begin{equation}
c n'(s,a) = n(s',a')
\end{equation}
where $(s',a')$ is a point in the interval $[{s_{i - 1}},{s_i}) \times [{a_{j - 1}},{a_j})$. Then, by substituting Eq\eqref{eq14} into Eq\eqref{eq13}, it can be inferred that  
\begin{equation}
\label{eq16}
 c u'(s,a) = u(s',a')
\end{equation}
That is, the uncertainty $u(s',a')$ is equal to the estimated uncertainty $u(s,a)$ obtained by the interval $[{s_{i - 1}},{s_i}) \times [{a_{j - 1}},{a_j})$ where $(s',a')$ is located in.
Since $u'$ is continuous, when the interval is taken as small enough, it can be inferred that:
\begin{equation}
\begin{array}{c}
\forall \varepsilon  > 0,\exists {\varepsilon’}>0\\ while \left| {({s_i},{a_j}) - ({s_i},{a_{j - 1}})} \right| +  \left| {({s_{i}},{a_j}) - ({s_{i-1}},{a_j})} \right| < {\varepsilon’}\\
\left| {u(s,a) -  c  u'(s',a')} \right| < \varepsilon 
\end{array}
\end{equation}
Therefore, for the uncertainty in any $(s,a)$, it can be estimated using the corresponding $c u'(s',a')$.
Finally, since ${\Gamma ^{\rm lcb}}(s,a)$ is proved to be continuous and an appropriate uncertainty quantization, it can replace $u(s,a)$ above. Substitute Eq\eqref{eq16} into Eq\eqref{lcb} to obtain:
\begin{equation}
P\left( {|\hat B{V_{t + 1}}(s,a) - B{V_{t + 1}}(s,a)| \le  c  u'(s,a) + \varepsilon } \right) \ge 1 - \xi 
\end{equation}
That is, an appropriate uncertainty constraint can be obtained by GPC.
\end{proof}
Theorem \ref{the1} shows that the uncertainty results obtained in \ref{4.1} can accurately reflect the real uncertainty, which provides a theoretical guarantee for the results. Compared with other count-based algorithms, the theoretical support of GPC is more intuitive and simple.

\subsection{GPC-SAC Algorithm}
\begin{algorithm}[tb]

    \caption{GPC-SAC algorithm}
    \label{alg1}
    \textbf{Input}: Initial policy parameter ${\it \theta}$; Q-function parameters $\phi_{1}$,$\phi_{2}$; Target Q-function parameters$\phi_{1}^{-}$, $\phi_{2}^{-}$; offline dataset ${\it D}$; The number of times the data is selected ${\it n(s,a)}$.\\
    \textbf{Output}: Trained policy 
    \begin{algorithmic}[1] 
        \WHILE{total iteration number ${\le}$ maximum iteration number}
	   \WHILE{total step number ${\le}$ maximum step number}
        \STATE Sample a mini-batch of $k$ state-action pairs $\left\{(s^t_{\rm in} ,a^t_{\rm in}),t=1,..,k \right\}$ from ${D}$
	   \STATE Calculate the OOD state-action pairs $(s^t_{\rm ood} ,a^t_{\rm ood})$
        \STATE Calculate the count-based uncertainty $u(s^t_{\rm ood},a^t_{\rm ood})$ according to Eq\eqref{gridOpertor}
        \STATE Calculate the Q-target according to Eq\eqref{inloss}, Eq\eqref{oodloss}
        \STATE Update Q-function with gradient descent as follows:
        
$\begin{array}{l}
\nabla {\phi _i}\frac{1}{{\left| B \right|}}\sum\limits_{t = 1}^k ({{{({Q_{{\phi _i}}}(s_{\rm in}^t,a_{\rm in}^t) - {B^{\rm in}}{Q_{{\phi _i}}}(s_{\rm in}^t,a_{\rm in}^t))}^2}} \\
 + {({Q_{{\phi _i}}}(s_{\rm in}^t,a_{\rm in}^{\rm ood}) - {B^{\rm in}}{Q_{{\phi _i}}}(s_{\rm in}^t,a_{\rm in}^{\rm ood}))^2}\\
 + {({Q_{{\phi _i}}}(s_{\rm next}^t,a_{\rm next}^{\rm ood}) - {B^{\rm ood}}{Q_{{\phi _i}}}(s_{\rm next}^t,a_{\rm next}^{\rm ood}))^2})
\end{array}$
	   \STATE Update policy with gradient ascent with Eq\eqref{policyupdt}
        \STATE Update target networks with:
			$ \phi_i^-=\rho \phi_i^-+(1-\rho)\phi_i  $
        \ENDWHILE
	   \ENDWHILE
        \STATE \textbf{return} solution
    \end{algorithmic}
\end{algorithm}
In order to maximize the power of GPC for OOD state-action constraints, we use the SAC algorithm, which is the most exploratory, as our benchmark framework and propose a new algorithm, GPC-SAC. The specific implementation of the pseudo-code of GPC-SAC is shown as algorithm \ref{alg1}.  

In GPC-SAC, due to the characteristics of the count-based method and deep offline RL, it is possible to have an underestimation of the Q-value of some OOD state-actions in the early training steps, so we use $\max\{B^{ood}Q^{k}_{\phi}(s_{next},a^{ ood}_{ next}),0\}$ and $\max\{B^{ ood}Q^{k}_{\phi}(s_{ in},a^{ ood}_{ in}),0\}$ to enable training to be stabilized at early training stages.

The uncertainty estimates obtained by GPC are used in GPC-SAC to constrain the Q-value and then the constrained Q-value is used to guide the agent to improve its policy. This operation only takes a short time and can improve the performance of SAC. It can also proved that the Q-value of the policy selected by GPC-SAC is singularly increasing, which ensures the performance of GPC-SAC. Since the proof is similar to the proof in sac, we put the proof in the Appendices.

\section{Experiment}
We considered different tasks in the D4RL benchmark~\cite{fu2020d4rl}, including Gym, Maze2d and Adroit. Gym is the most commonly used evaluation criterion task these days, with a relatively smooth reward function. Therefore, Gym is relatively simple compared to other tasks. Maze2d navigation problem using sparse rewards requires the agent to obtain the optimal path from multiple trajectories, making it more challenging. Adroit is mainly used for complex human activities, and the data in the dataset only reflects a small part of the state-actions, so the agent need to learn from limited state-actions.

\subsection{Experiment in Gym}
\subsubsection{Experiment setting}

\begin{table*}[!h]
\centering
\caption{The standardized average performance and standard deviation of each algorithm on the Gym dataset. In the table, half represents halfcheetah; hp represents hopper; w represents walker2d; r represents random; m represents medium and e represents expert. Due to the lack of open source code for TD3-CVAE, the corresponding experimental results are from Table 1 in~\cite{rezaeifar2022offline} with 10 random seeds. Due to the high computational cost of PBRL, PBRL is only tested in partial environments with 5 random seeds, while the rest of the results were obtained from \protect\cite{bai2022pessimistic}. GPC-SAC are tested with 5 random seeds.}
\begin{tabular}{cccccccc}
\toprule
   Env  	 & IQL  & UWAC & CQL & PBRL & TD3-CVAE & \textbf{GPC-SAC} \\
\midrule
       half-r         & $11.0 {\pm} 2.5$    & $2.3 {\pm} 0.0$    & $28.3 {\pm} 0.5$      & $11.0 {\pm} 5.8$    & $28.6 {\pm} 2.0$ & $\textbf{31.0} {\pm} \textbf{1.8}$     \\
        hp-r         & $7.8 {\pm} 0.4$   & $2.7 {\pm} 0.3$     & $16.4 {\pm} 14.5$      & $26.8 {\pm} 9.3$    & $11.7 {\pm} 0.2$ & $\textbf{31.4} {\pm} \textbf{0.0}$     \\
        w-r         & $6.4 {\pm} 0.1$    & $2.0 {\pm} 0.4$    & $4.2 {\pm} 0.4$      & $8.1 {\pm} 4.4$    & $5.5 {\pm} 8.0$ & $\textbf{9.6} {\pm} \textbf{12.0}$     \\
        half-m         & $47.4 {\pm} 0.2$  & $42.2 {\pm} 0.4$     & $47.0 {\pm} 0.5$      & $57.9 {\pm} 1.5$    & $43.2 {\pm} 0.4$ & $\textbf{60.8} {\pm} \textbf{0.7}$     \\
        hp-m              & $66.2 {\pm} 5.7$   & $50.9 {\pm} 4.4$     & $53.0 {\pm} 28.5$     & $75.3 {\pm} 31.2$   & $55.9 {\pm} 11.4$  & $\textbf{82.9} {\pm} {\textbf{2.3}}$    \\
        w-m           & $78.3 \pm 8.7$    & $75.4 \pm 3.0$    & $73.3 \pm 17.7$   & $\textbf{89.6} \pm \textbf{0.7}$  & $68.2 \pm 18.7$    & $87.6 \pm 1.3 $    \\
        half-m-r  & $44.2 \pm 1.2$   & $35.9 \pm 3.7$     & $45.5 \pm 0.7$   & $45.1 \pm 9.8$    & $45.3 \pm 0.4$     & $\textbf{55.7} \pm \textbf{1.0}$    \\
        hp-m-r       & $94.7 \pm 8.6$  & $25.3 \pm 1.7$  & $35.9 \pm 3.7$     & $88.7 \pm 12.9$   & $\textbf{100.6} \pm \textbf{1.0}$     & $97.5 \pm 3.6$    \\
        w-m-r     & $73.8 \pm 7.1$   & $23.6 \pm 6.9$     & $81.8 \pm 2.7$      & $77.7 \pm 14.5$   & $15.4 \pm 7.8$  & $\textbf{86.2} \pm \textbf{2.4}$    \\
        half-m-e  & $86.7 \pm 5.3$     & $42.7 \pm 0.3$   & $75.6 \pm 25.7$     & $\textbf{92.3} \pm \textbf{1.1}$ & $86.1 \pm 9.7$      & $87.1 \pm 5.8$    \\
        hp-m-e       & $91.5 \pm 14.3$   & $44.9 \pm 8.1$    & $105.6 \pm 12.9$    & $110.8 \pm 0.8$   & $\textbf{111.6} \pm \textbf{2.3}$ & $109.8 \pm 3.4$   \\
        w-m-e     & $109.6 \pm 1.0$   & $96.5 \pm 9.1$    & $107.9 \pm 1.6$     & $110.1 \pm 0.3$   & $84.9 \pm 9.7$  & $\textbf{111.7} \pm \textbf{1.0}$    \\
        half-e         & $95.0 \pm 0.5$   & $92.9 \pm 0.6$     & $96.3 \pm 1.3$      & $92.4 \pm 1.7$    & ---         & $\textbf{104.7} \pm \textbf{2.2}$      \\
        hp-e              & $109.4 \pm 0.5$ & $\textbf{110.5} \pm \textbf{0.5}$  & $96.5 \pm 28.0$  & $\textbf{110.5} \pm \textbf{0.4}$  & ---   & $109.5 \pm 4.0$    \\   
        w-e            & $109.9 \pm 1.0$  & $108.4 \pm 0.4$  & $108.5\pm 0.5$      & $108.3 \pm 0.3$   & ---        & $\textbf{111.7} \pm \textbf{1.0}$     \\
        Average                    & $68.8 \pm 3.8$   & $50.4 \pm 2.7$  & $68.6 \pm 9.9$     & $74.4 \pm 5.3$    &  $50.3 \pm 7.4$ & $\textbf{78.5} \pm \textbf{2.8}$      \\
\bottomrule
\end{tabular}
\label{mainres}
\end{table*}

To validate the effectiveness of GPC-SAC, experiments are performed in Gym dataset. Three types of environments are included in Gym dataset: halfcheetah, hopper and walker2d. Random, medium, medium-replay, medium-expert and expert five different policies are used in each type of environment to collect samples. Specifically, random means completely random sampling to obtain data. Medium means training an algorithm online, pausing training midway, and then utilizing this partially trained policy for data collection. Medium-replay means training an algorithm online until the policy reaches a medium performance level and collecting all samples saved in the buffer during training. Medium-expert means mixing the same amount of data by expert policy and medium policy. Expert means training a policy online to expert performance level and using this expert policy for sample collection. In most environments, 1000 episodes are trained and 5 different random seeds are used for testing. For detailed parameter settings, please refer to \ref{implement detail}.

To demonstrate the effectiveness of GPC-SAC, GPC-SAC is compared with several classical algorithms and SOTA algorithms. The involved algorithms are: IQL~\cite{Kostrikov2021OfflineRL}, which proposed an in-sample Q-learning algorithm that does not use OOD samples at all, UWAC~\cite{wu2021uncertainty}, an uncertainty estimation algorithm that improves BEAR~\cite{kumar2019stabilizing} through dropout, CQL~\cite{kumar2020conservative}, an algorithm that constrains the policy by conservatively estimating the Q-value of OOD samples through regularization terms, TD3-CVAE~\cite{rezaeifar2022offline}, another model-free domain algorithm that uses the count-based method to subtract the predicted uncertainty from the reward for conservative learning, and PBRL~\cite{bai2022pessimistic}, an algorithm in the model-free domain that quantifies uncertainty and 
constraints Q-values through the ensemble method.

\begin{figure*}[!htbp]
	\centering
    \subfigure[w-r]{
		\includegraphics[width=0.31\linewidth]{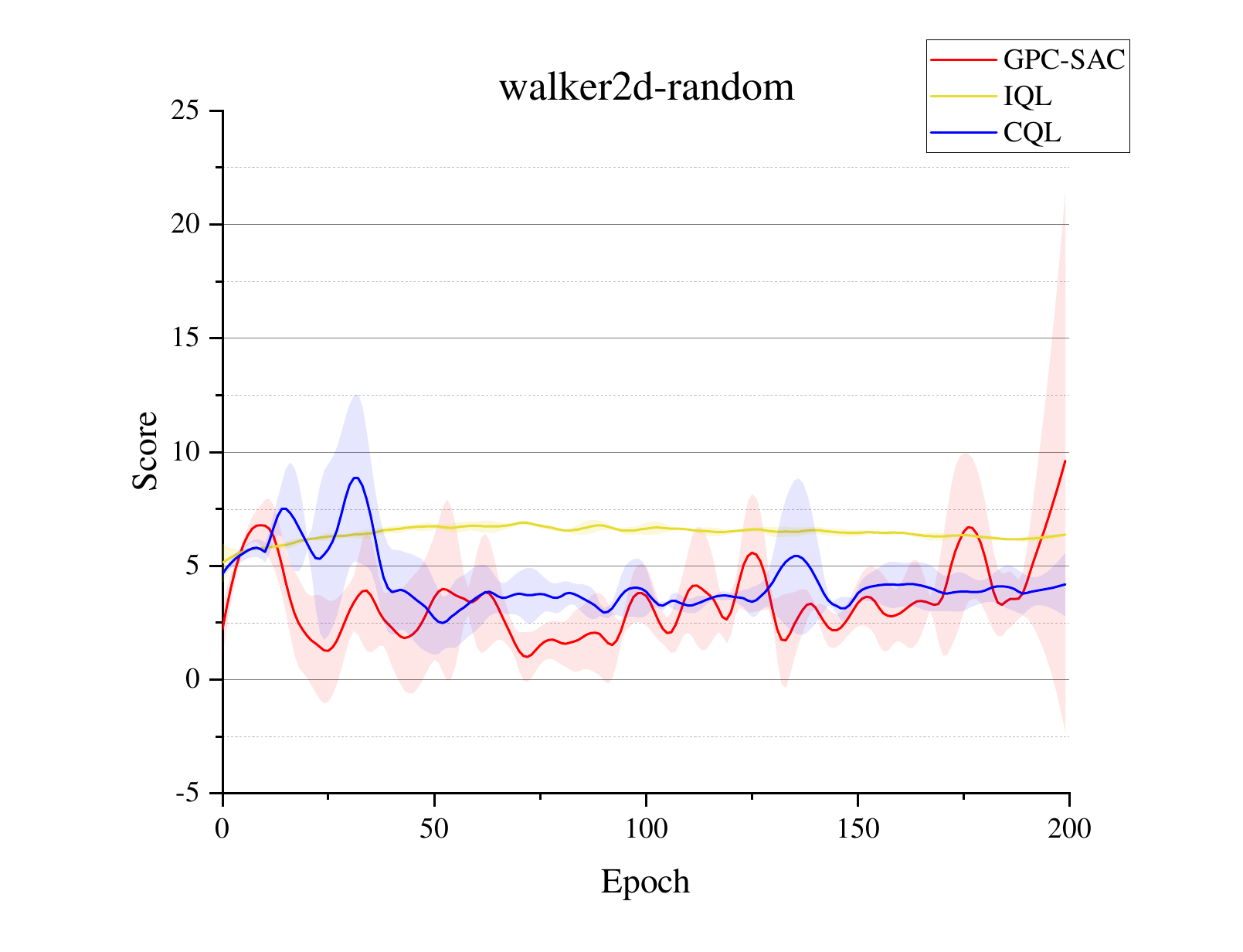}}
	\subfigure[w-m]{
		\includegraphics[width=0.31\linewidth]{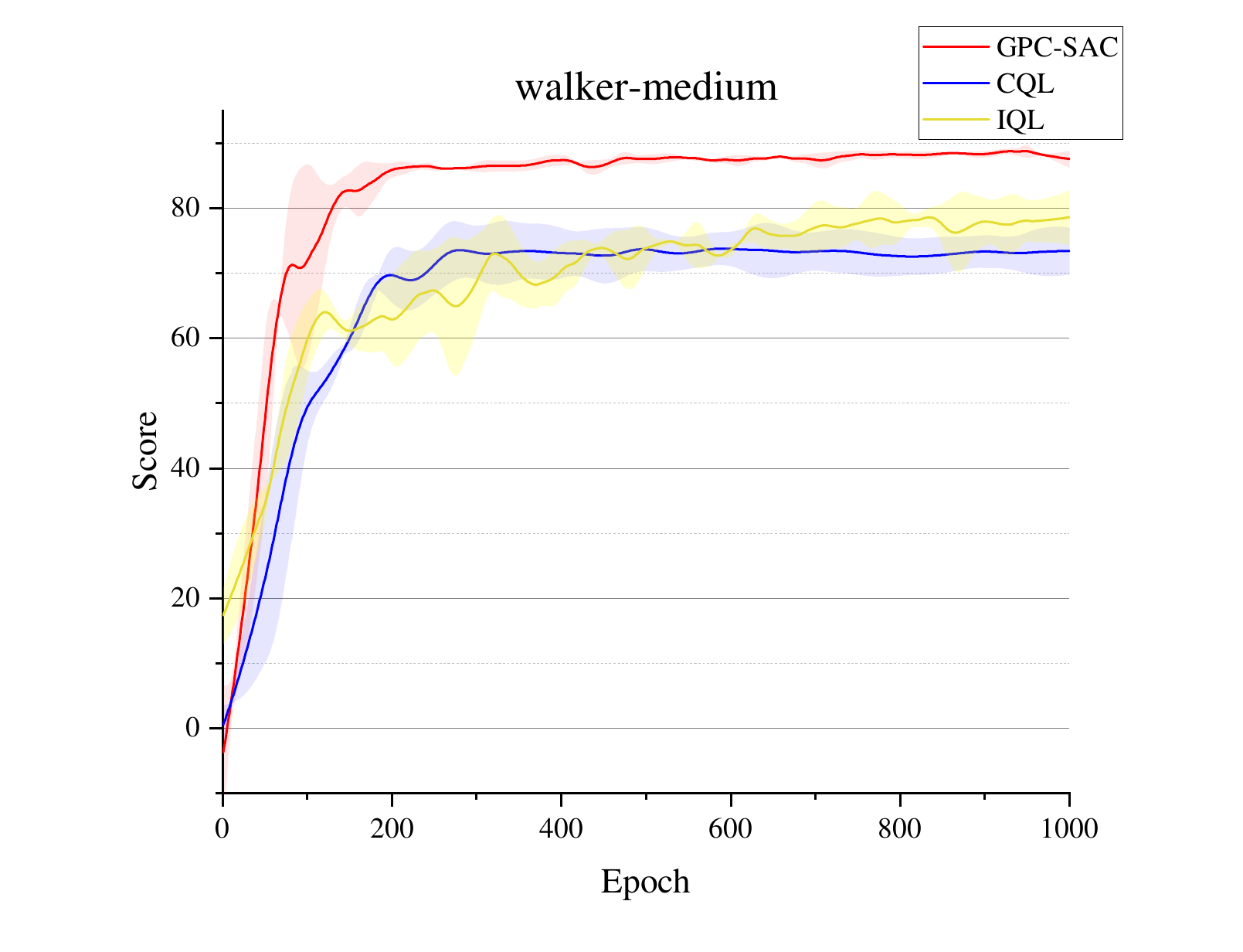}}
	\subfigure[w-m-r]{
		\includegraphics[width=0.31\linewidth]{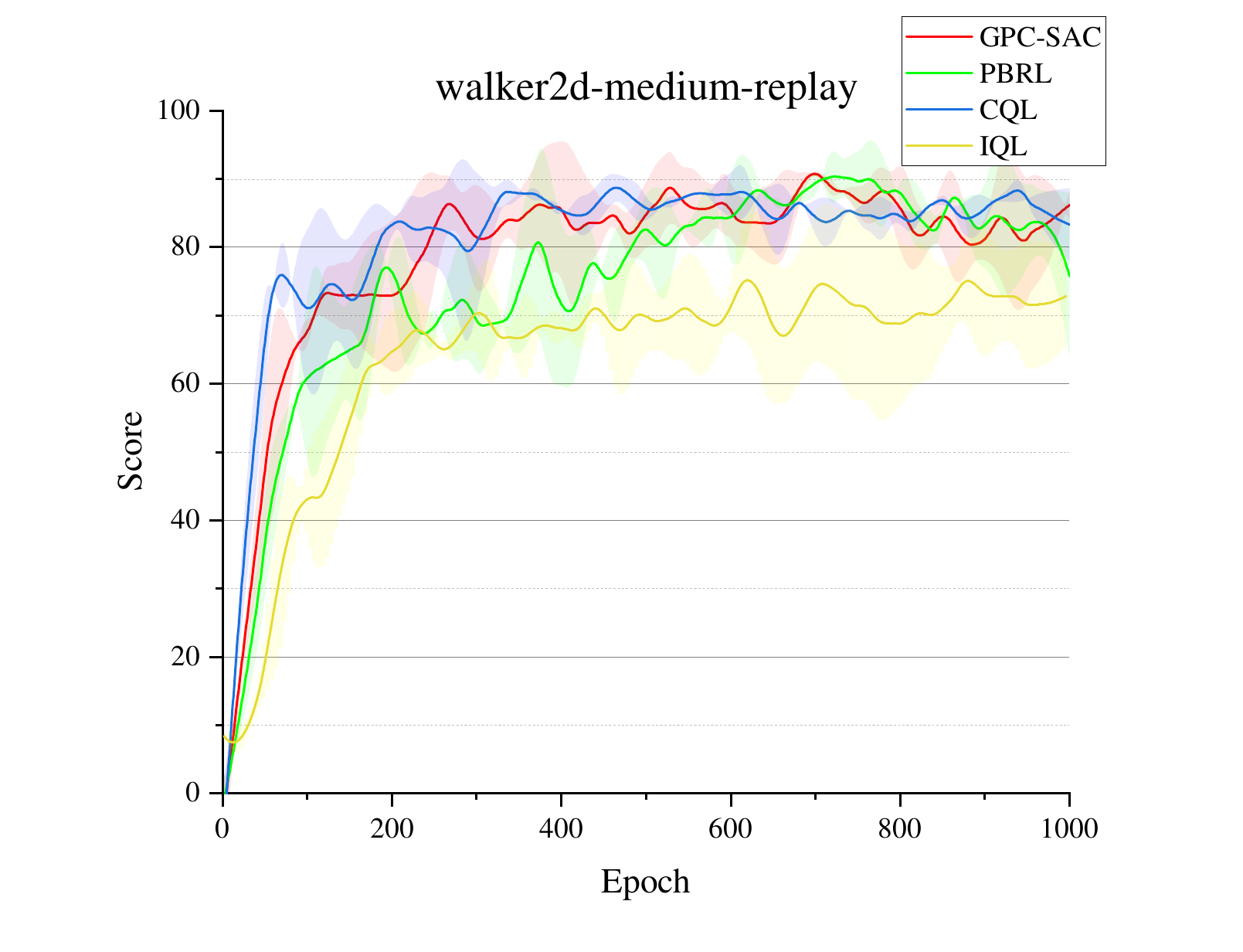}}
    \subfigure[w-m-e]{
		\includegraphics[width=0.31\linewidth]{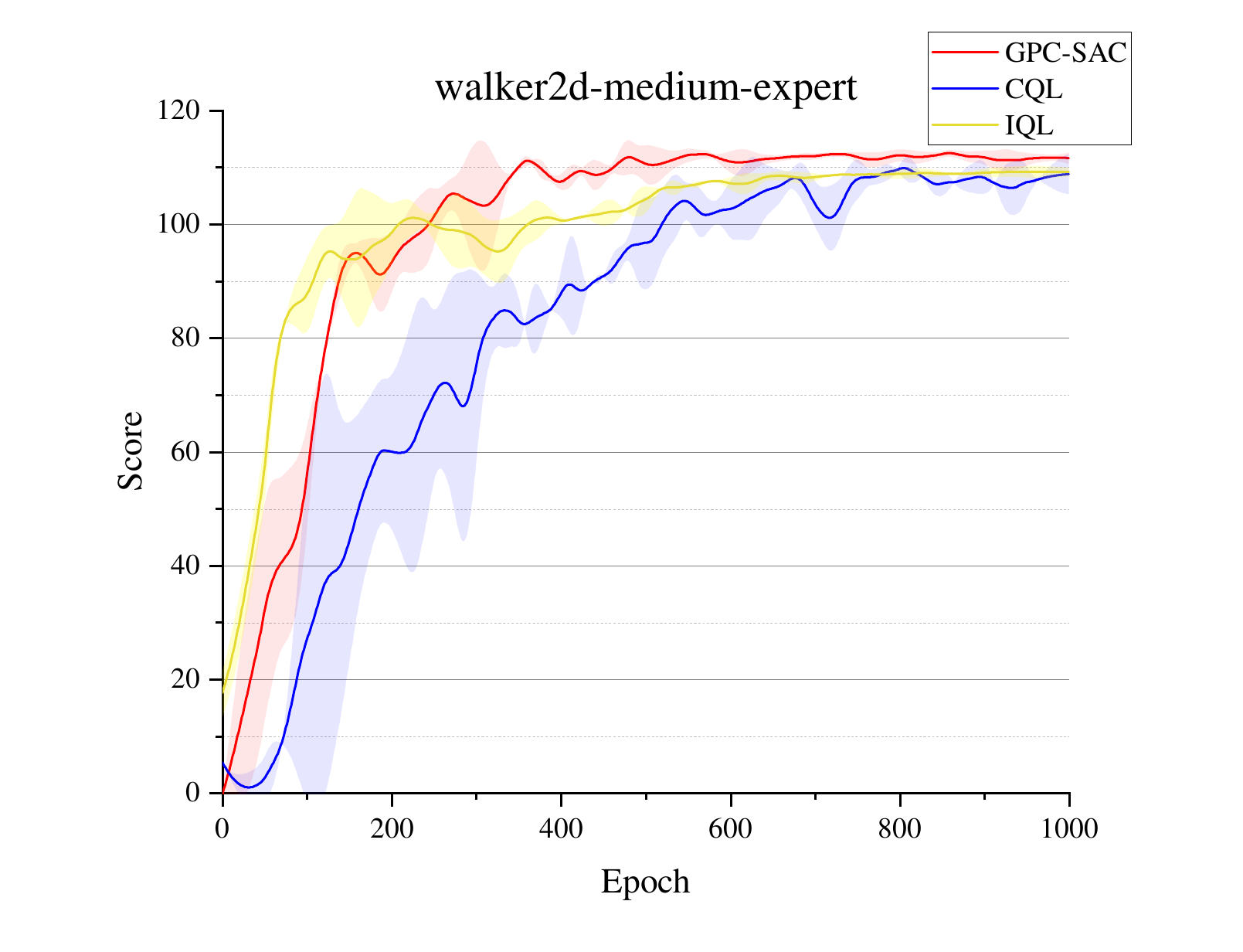}}
	\subfigure[w-e]{
		\includegraphics[width=0.31\linewidth]{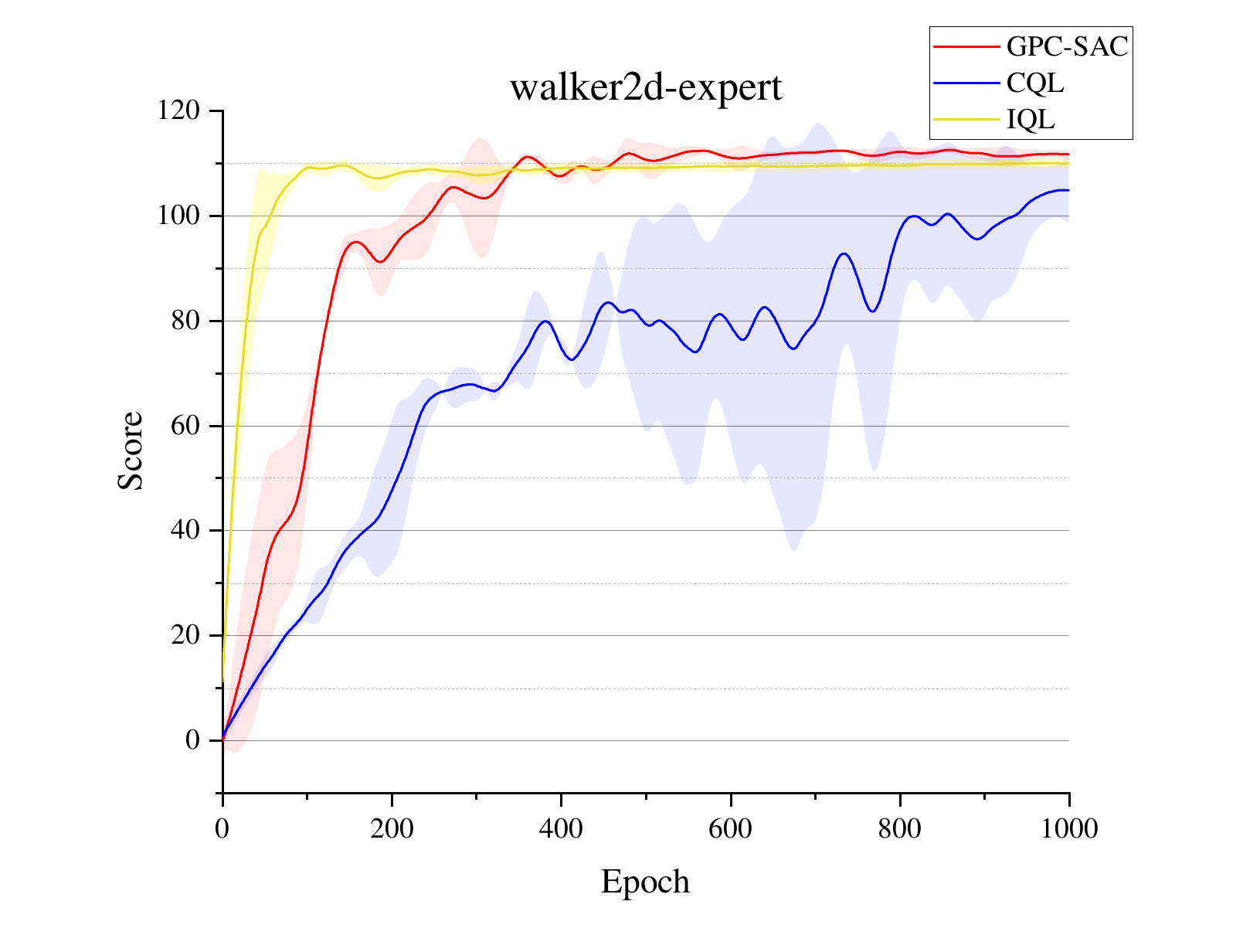}}
    \subfigure[half-r]{
		\includegraphics[width=0.31\linewidth]{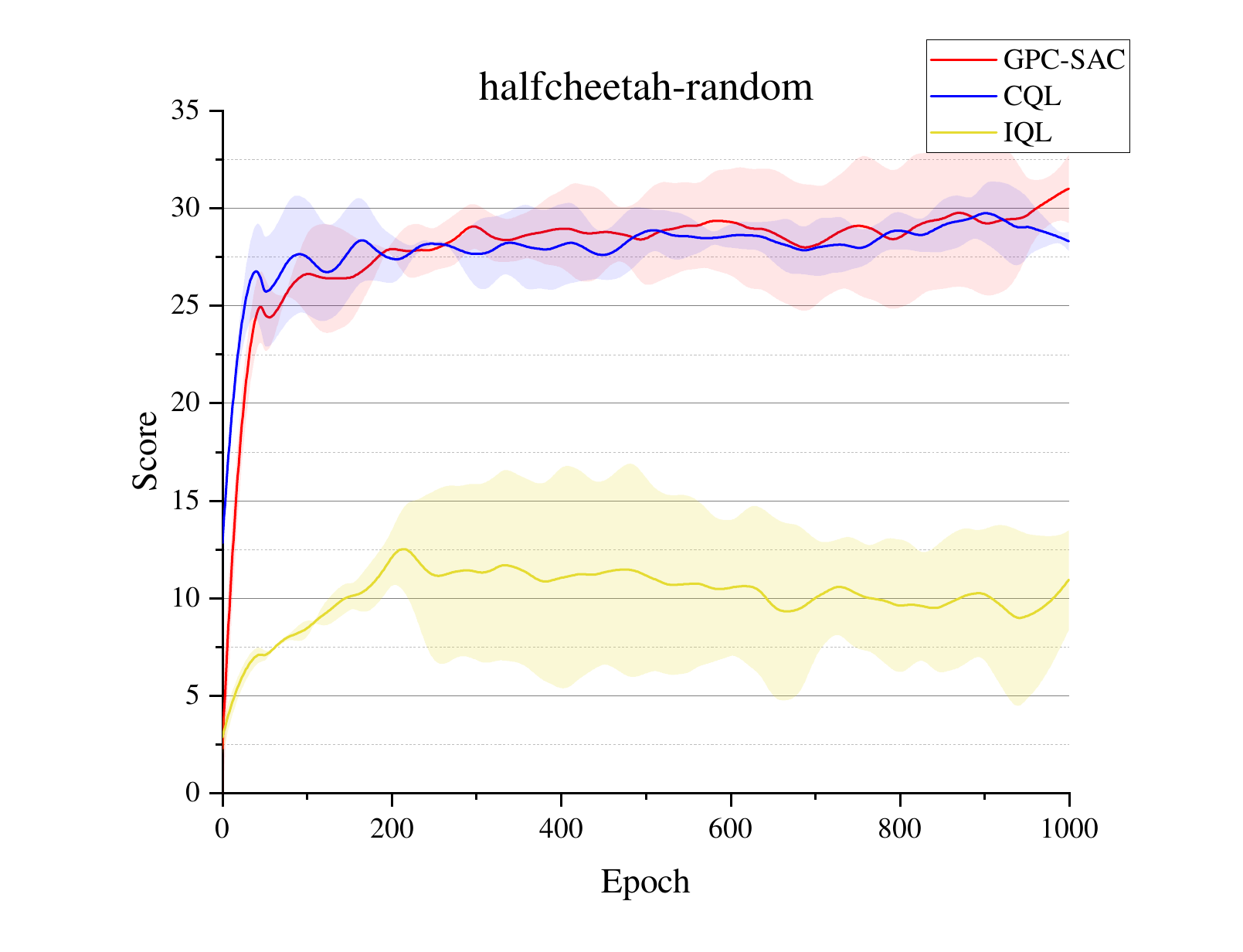}}
    \subfigure[half-m]{
		\includegraphics[width=0.31\linewidth]{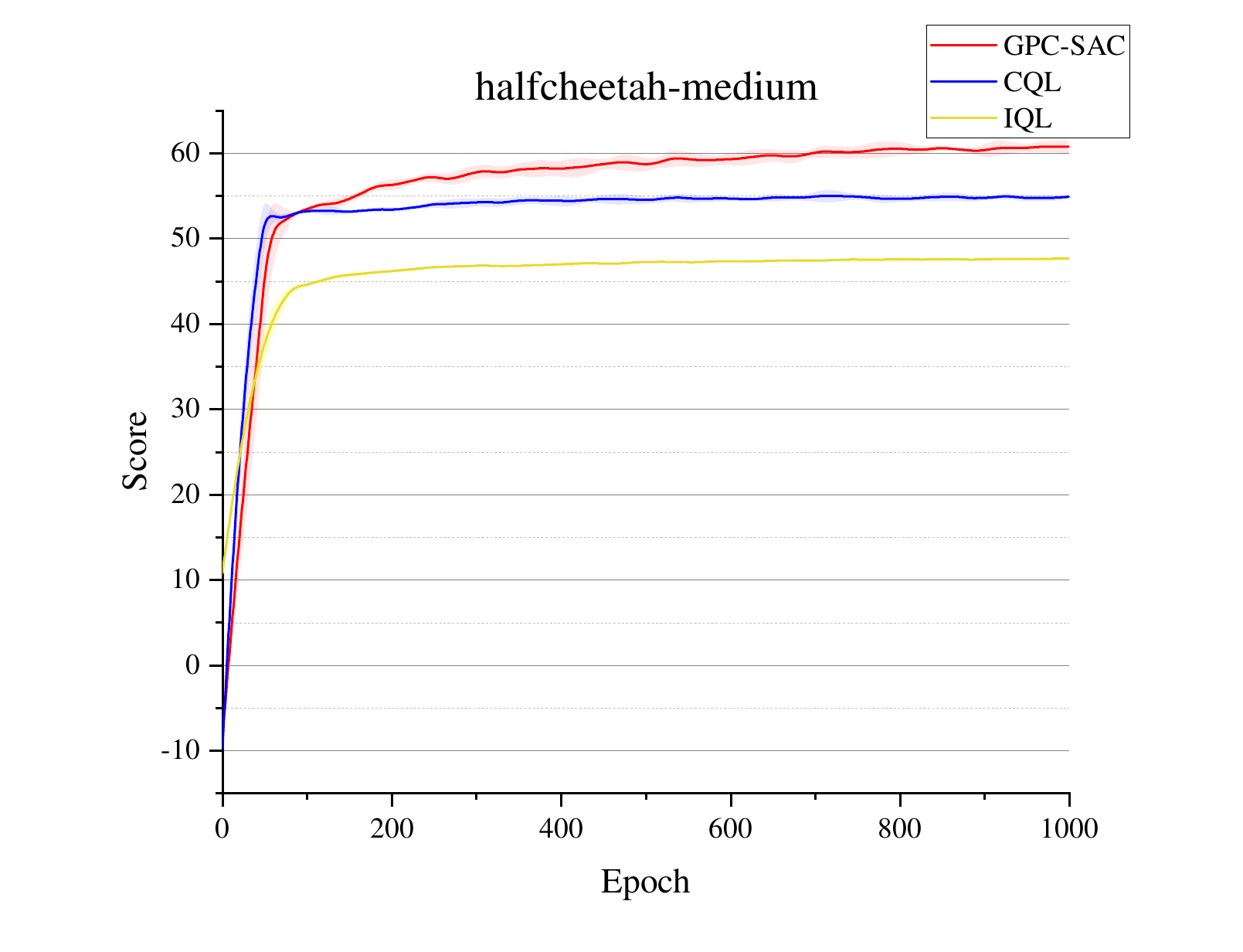}}
	\subfigure[half-m-r]{
		\includegraphics[width=0.31\linewidth]{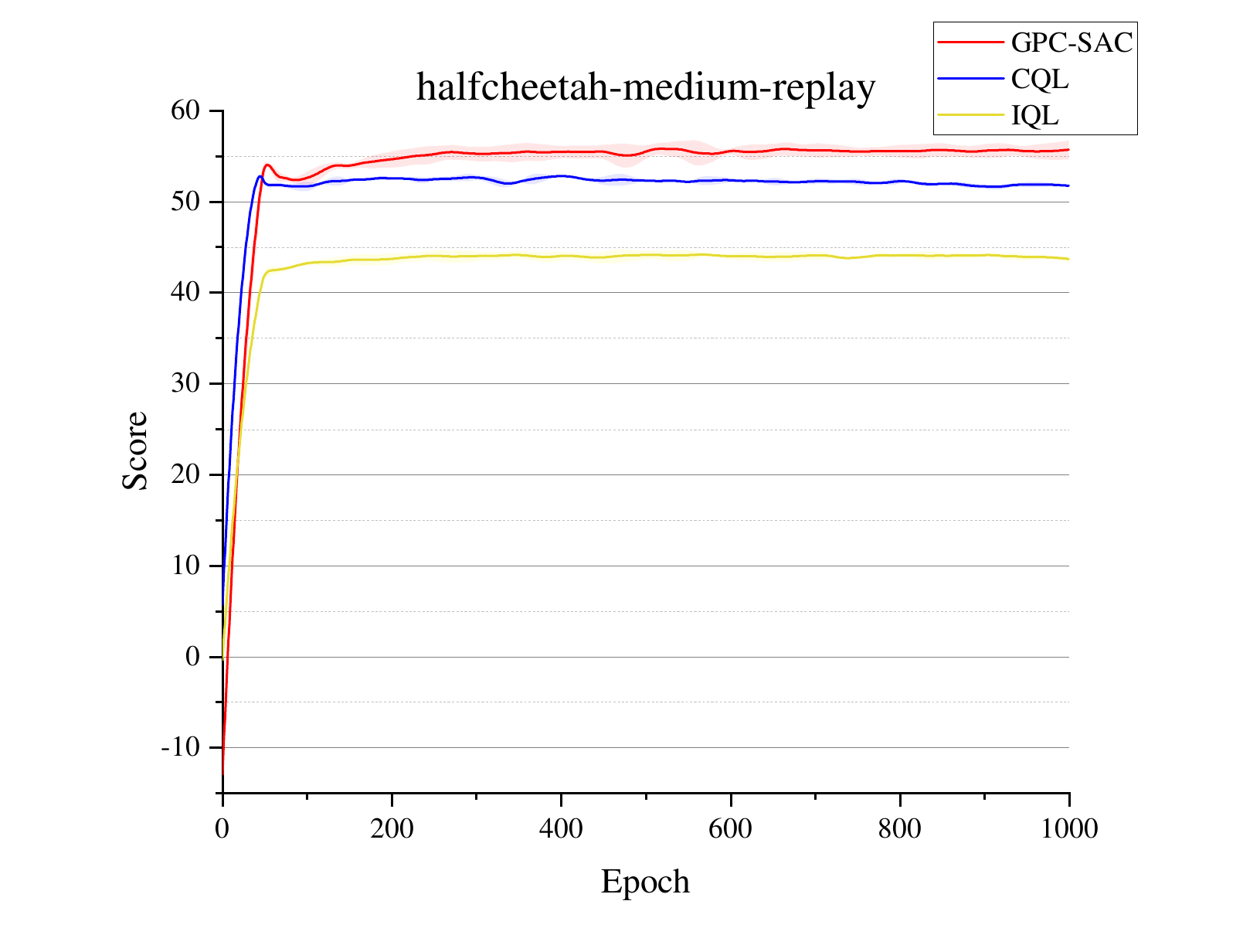}}
    \subfigure[half-m-e]{
		\includegraphics[width=0.31\linewidth]{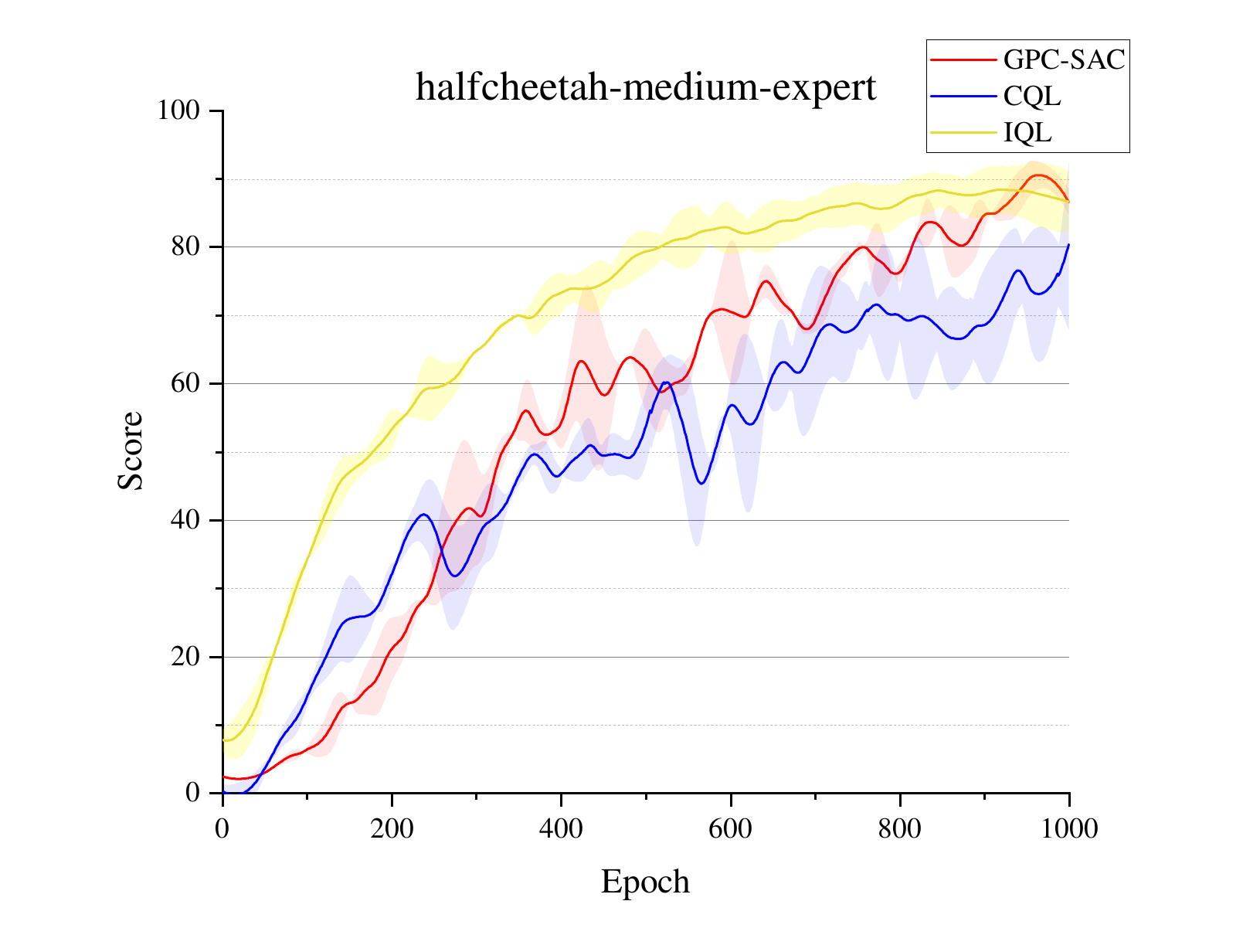}}
	\subfigure[half-e]{
		\includegraphics[width=0.31\linewidth]{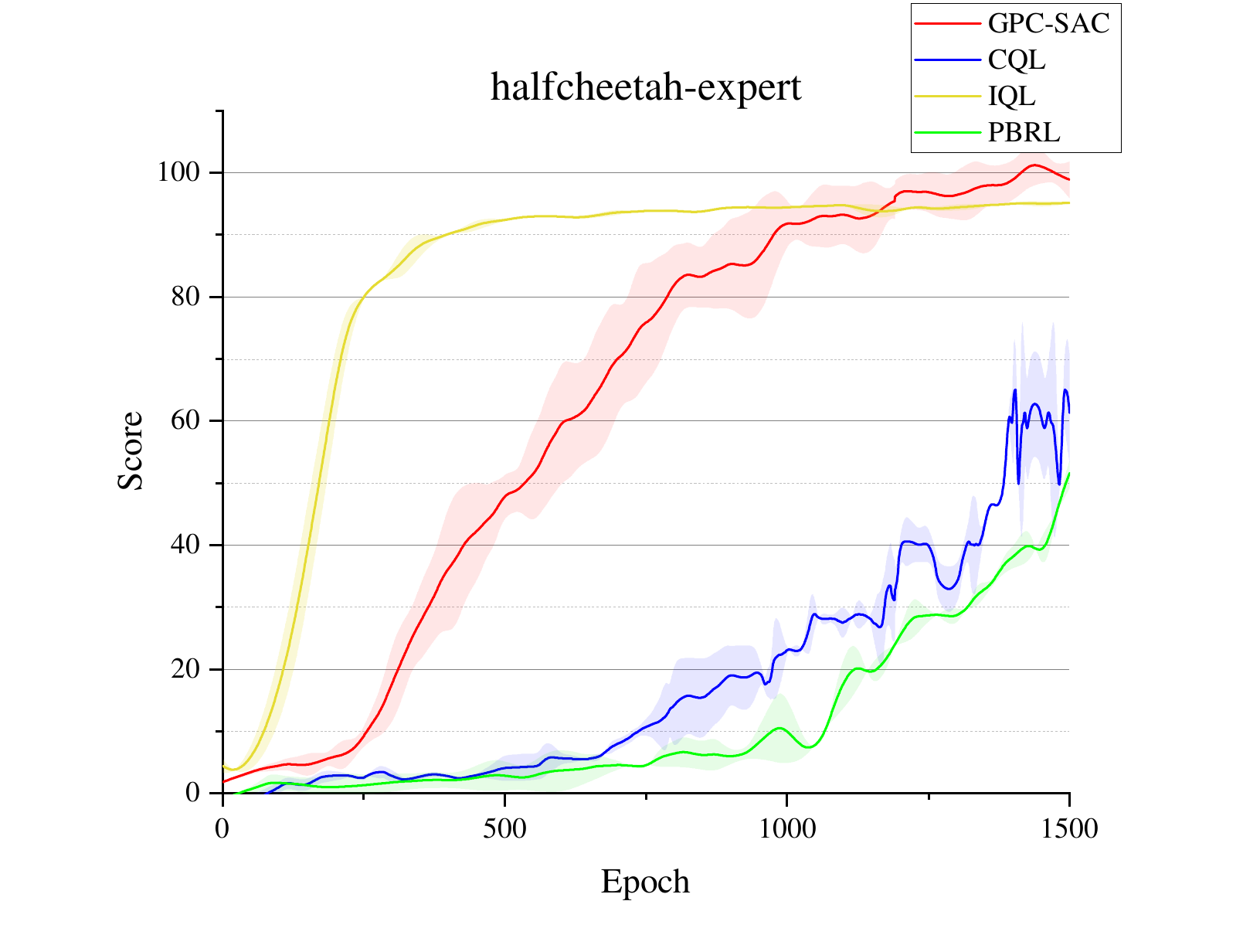}}
    \subfigure[hp-r]{
		\includegraphics[width=0.31\linewidth]{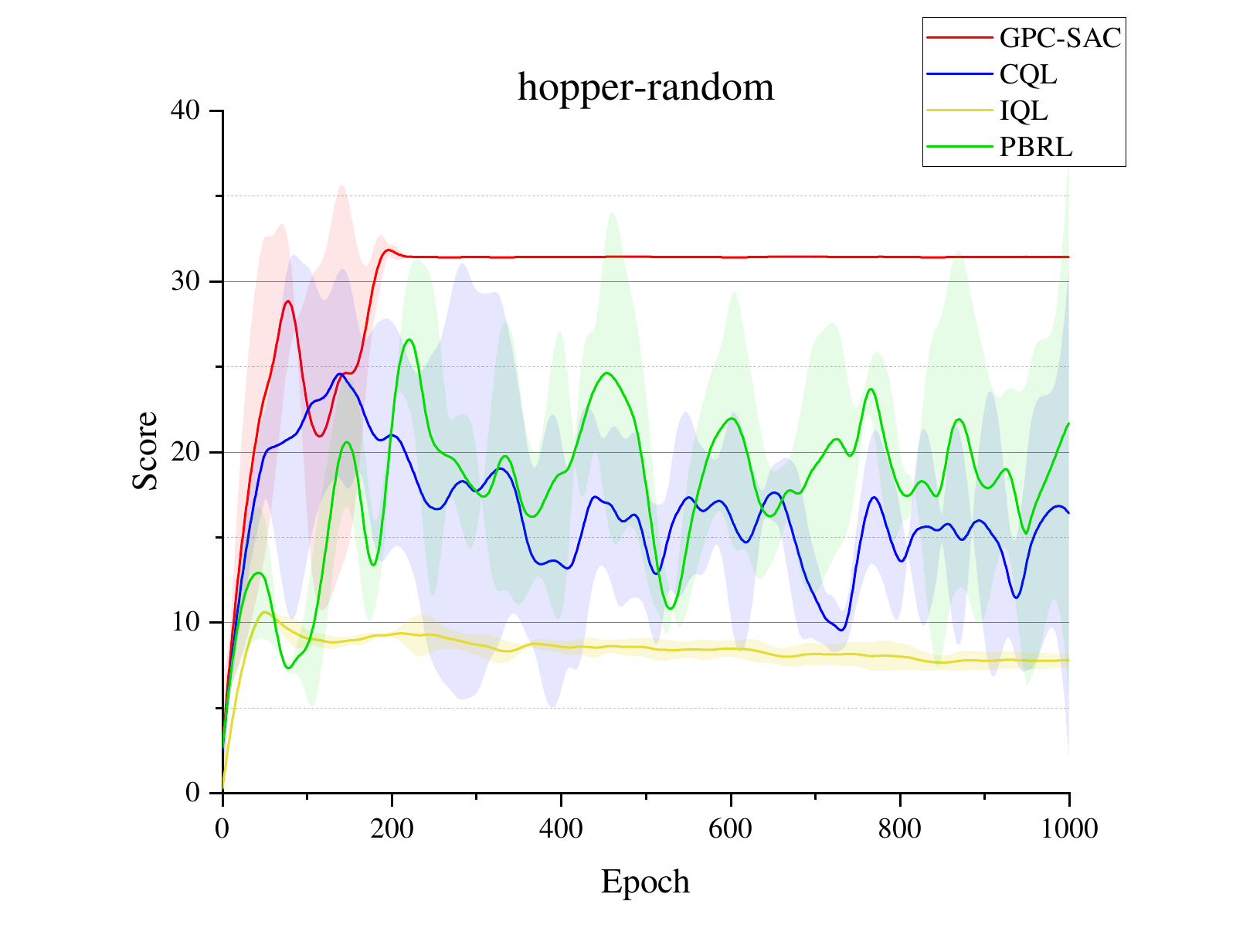}}
    \subfigure[hp-m]{
		\includegraphics[width=0.31\linewidth]{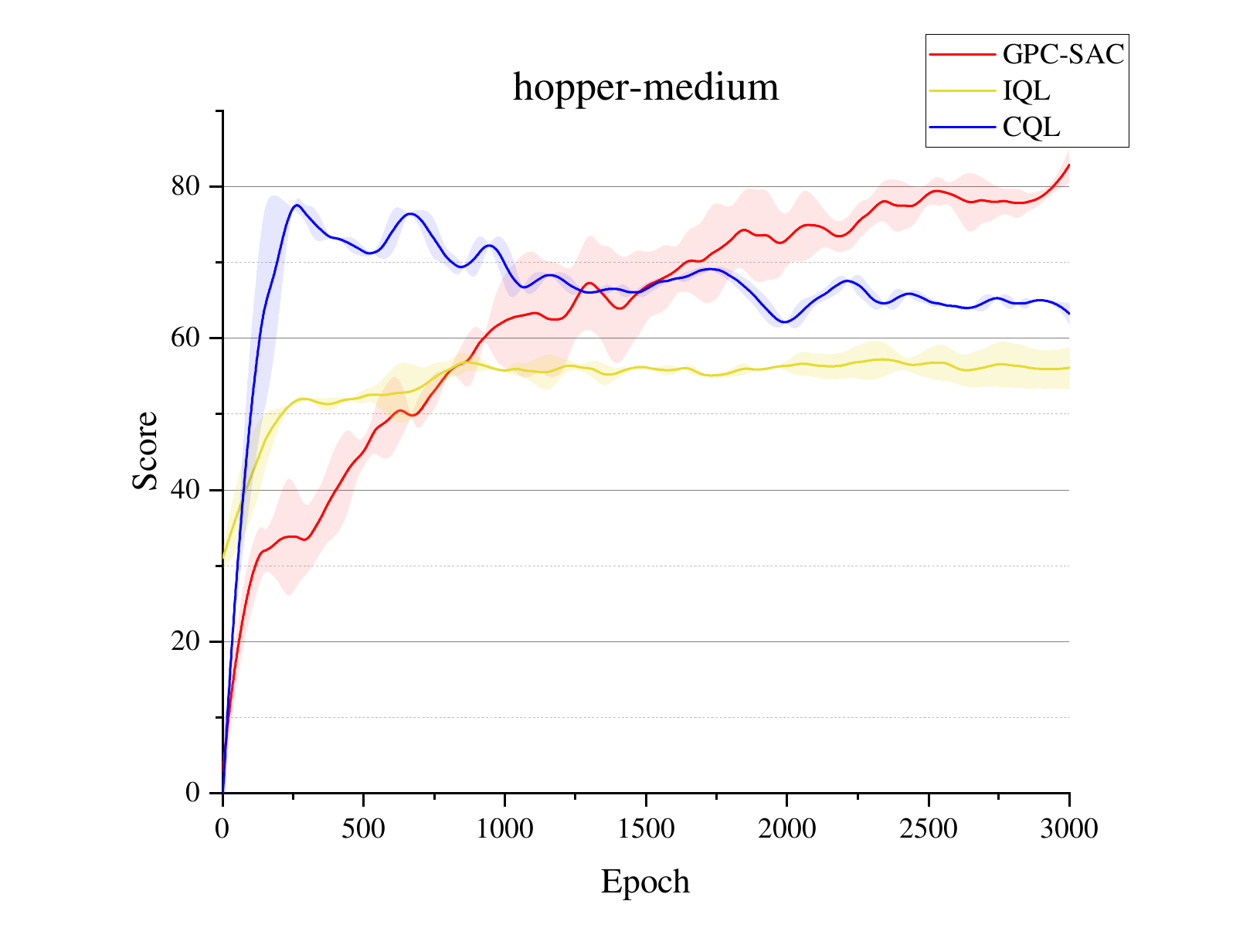}}
    \caption{Gym training curve}
    \end{figure*}
    \begin{figure*}[!htbp]
    \ContinuedFloat
    \centering
	\subfigure[hp-m-r]{
		\includegraphics[width=0.31\linewidth]{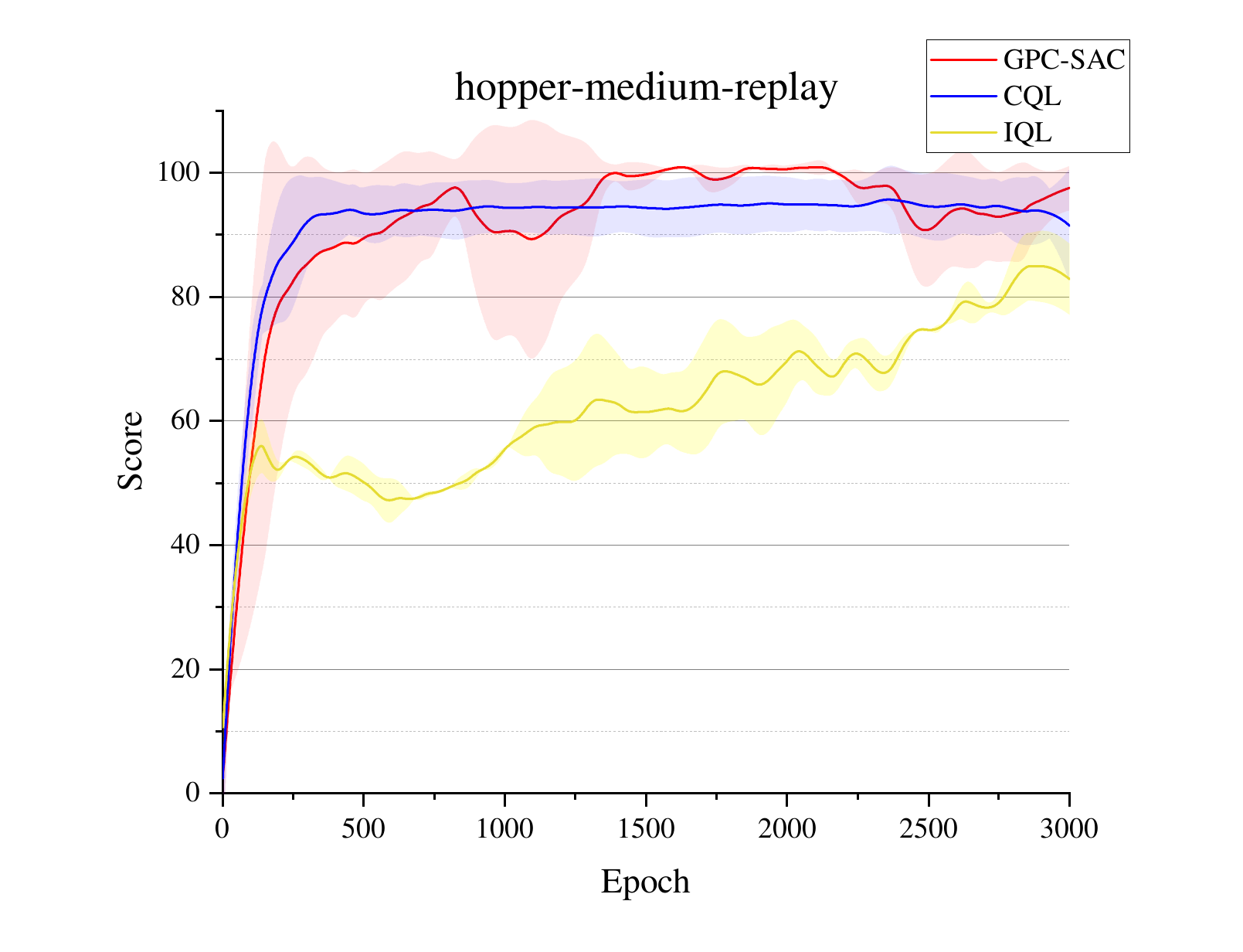}}
    \subfigure[hp-m-e]{
		\includegraphics[width=0.31\linewidth]{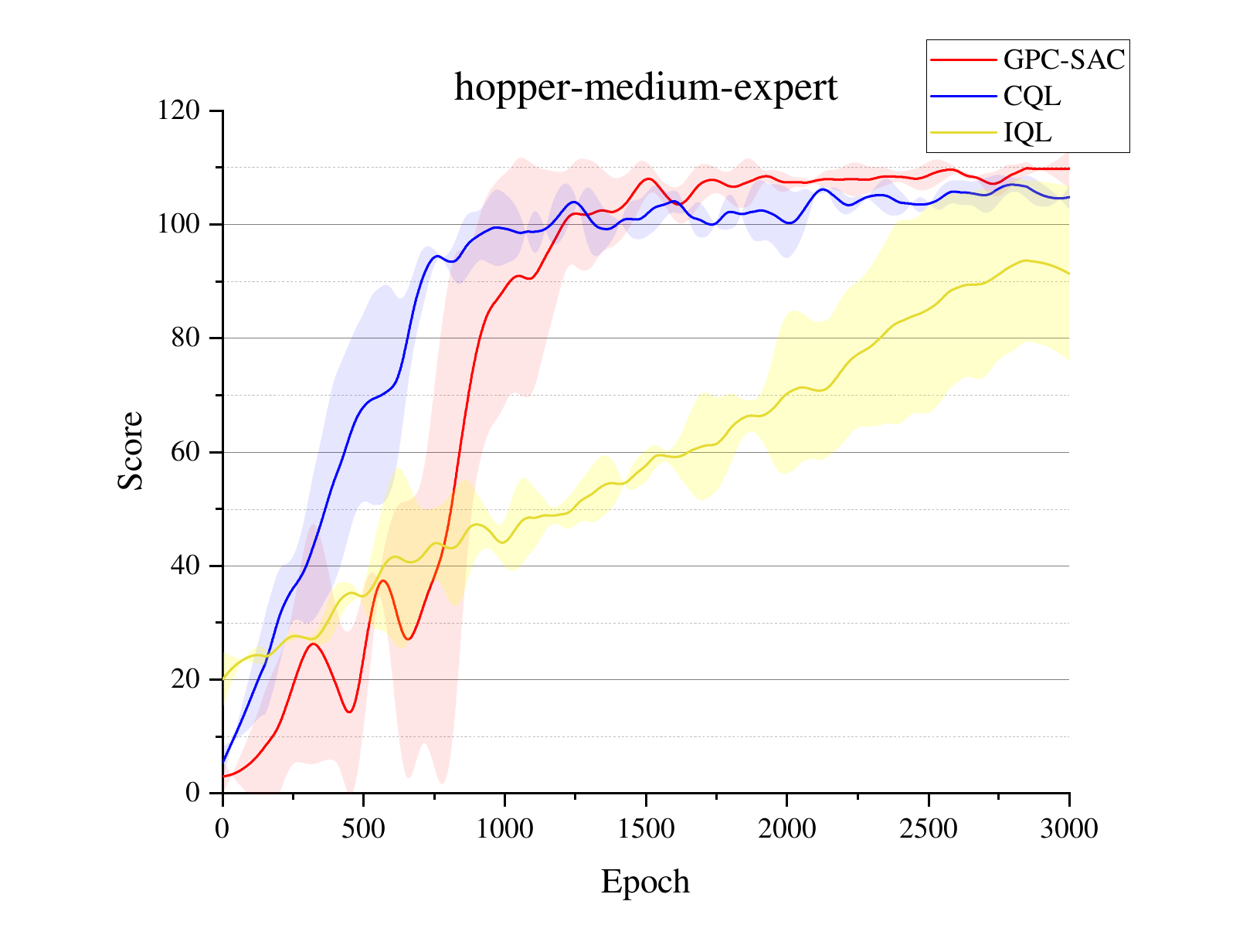}}
	\subfigure[hp-e]{
		\includegraphics[width=0.31\linewidth]{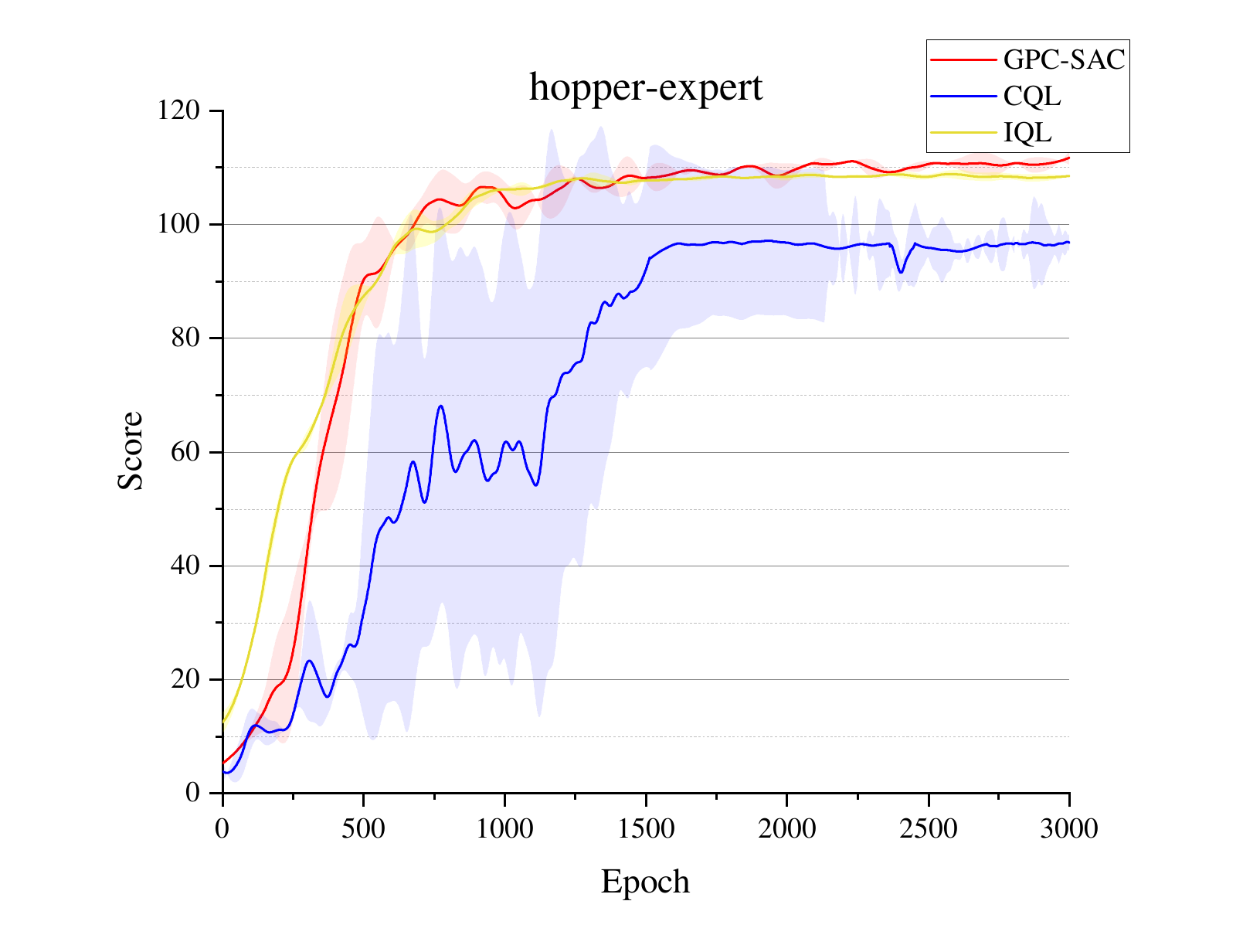}}
	\caption{Gym training curve}
 \label{exp}
\end{figure*}

\subsubsection{Result analysis}
From Table \ref{mainres}, it can be seen that the performance of GPC-SAC is the best in 10 environments, the second best in 3 environments and the third best in 2 environments. In hopper-random, halfcheetah-medium and halfcheetah-expert, the performance of GPC-SAC outperformed the highest-performing algorithm by 20\%. In halcheetah-medium-expert, where GPC-SAC performs relatively poorly, the performance of GPC-SAC is only 5\% lower than that of PBRL, the best-performing algorithm. In general, compared GPC-SAC with classic algorithms and SOTA algorithms, only PBRL has approximate performance. However, the computational cost of PBRL is much larger than GPC-SAC, which can be seen in Table \ref{traintime}. The experiment results show the effectiveness of GPC-SAC, and reasonable uncertainty constraints can be obtained through GPC-SAC with simple calculations. 

\begin{table}[!htbp]
\caption{Experimental results in maze2d and adroit tasks. Here u represents umaze, $m$ represents medium, l represents large, e represents expert.}
    \centering
    \begin{tabular}{ccccc}
        \toprule
        Env   & CQL & IQL & PBRL &\textbf{GPC-SAC}\\
        \midrule
        maze2d-u    & 56.3   & 46.9   &86.7 &\textbf{141.0}\\
        maze2d-m   & 24.8   & 32.0   &71.1  &\textbf{103.7}\\
        maze2d-l    & 15.3   & 64.2   &64.5  &\textbf{134.3}\\
        pen-e      & 107.0  & 117.2  &\textbf{137.7} &118.8 \\
        hammer-e   & 86.7   & 124.1  &\textbf{127.5} &95.6  \\
        door-e     & 101.5  & \textbf{105.2}  &95.7  &101.0\\
        average         &65.3    &81.6     &97.2   &\textbf{115.7}\\
        \bottomrule
    \end{tabular}
    \label{score_ad}
\end{table}

\begin{figure*}[!h]
	\centering
    \subfigure[maze2d-umaze]{
		\includegraphics[width=0.31\linewidth]{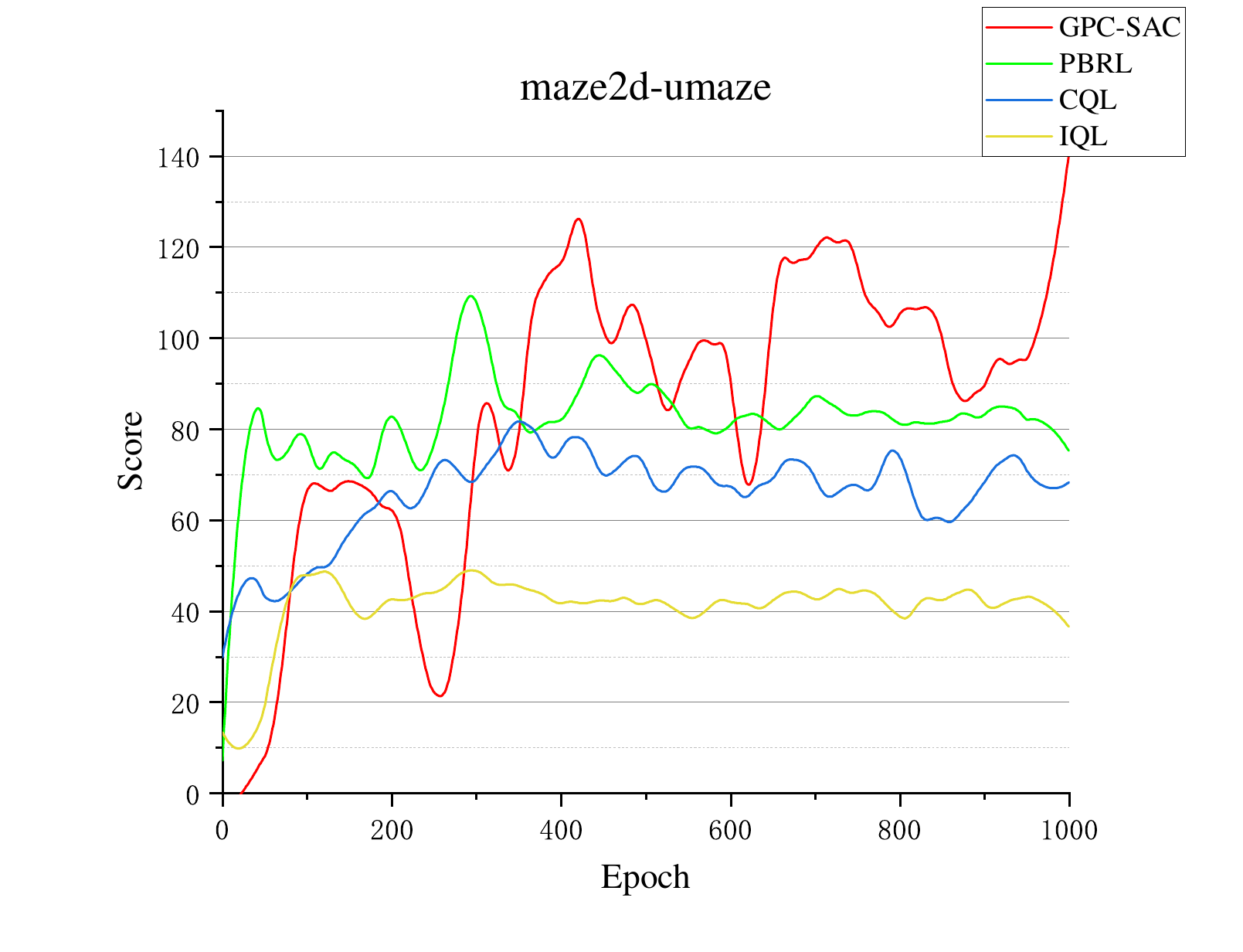}}
	\subfigure[maze2d-medium]{
		\includegraphics[width=0.31\linewidth]{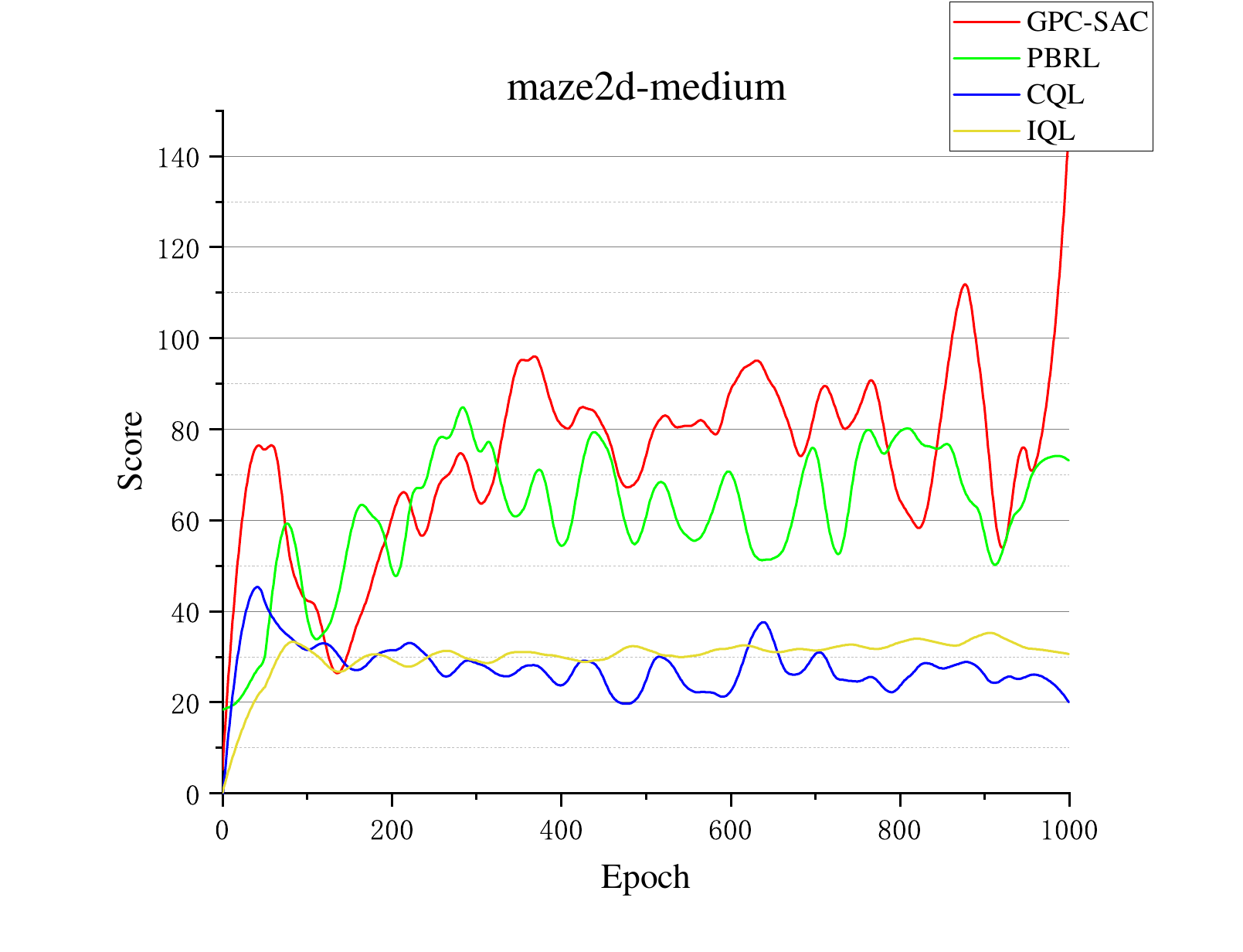}}
	\subfigure[maze2d-large]{
		\includegraphics[width=0.31\linewidth]{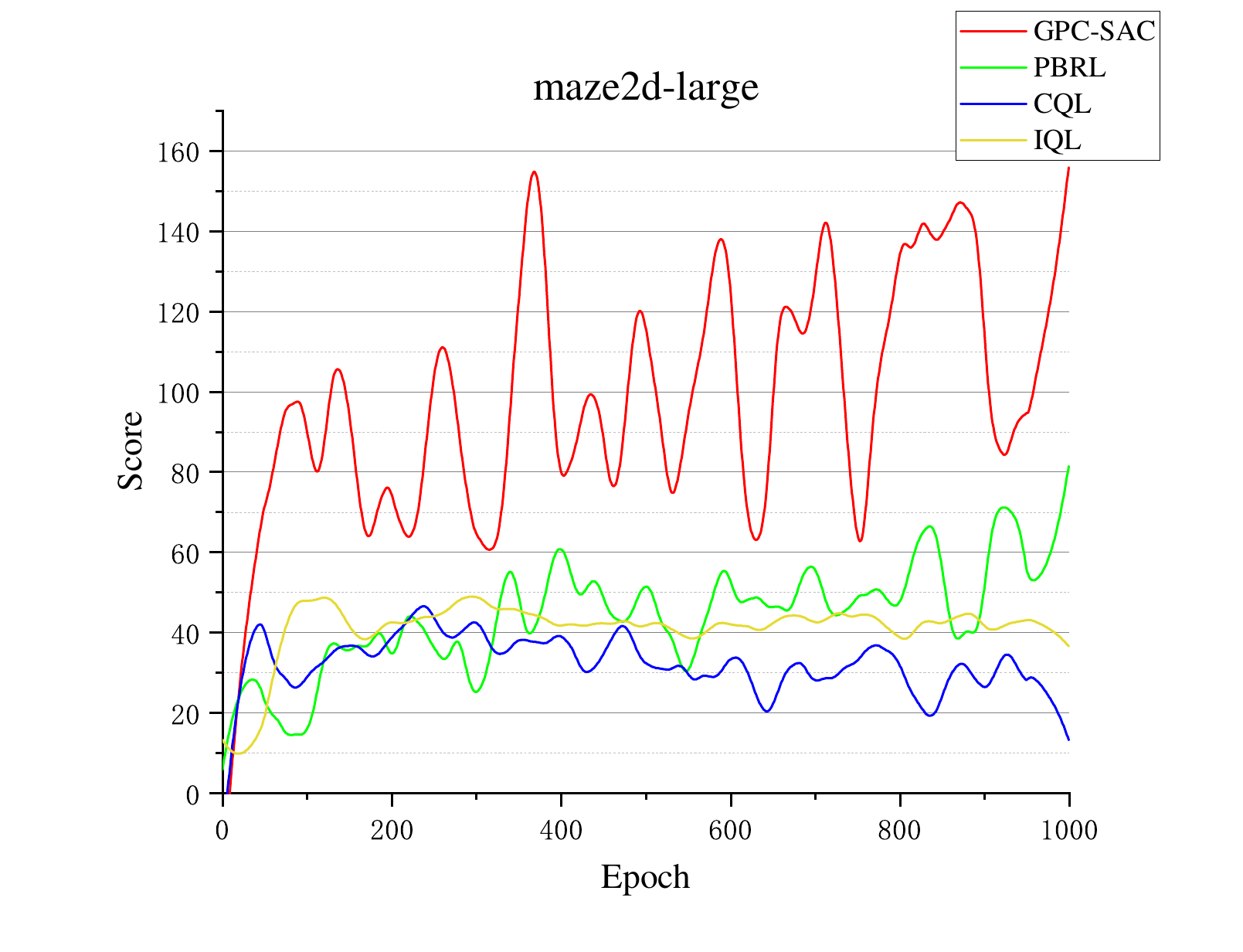}}
  \caption{maze2d training curve}\
\end{figure*}

\subsubsection{Experiment in Adroit}
\label{adroit}

\begin{figure*}[!h]
	\centering
    \subfigure[pen-expert]{
		\includegraphics[width=0.31\linewidth]{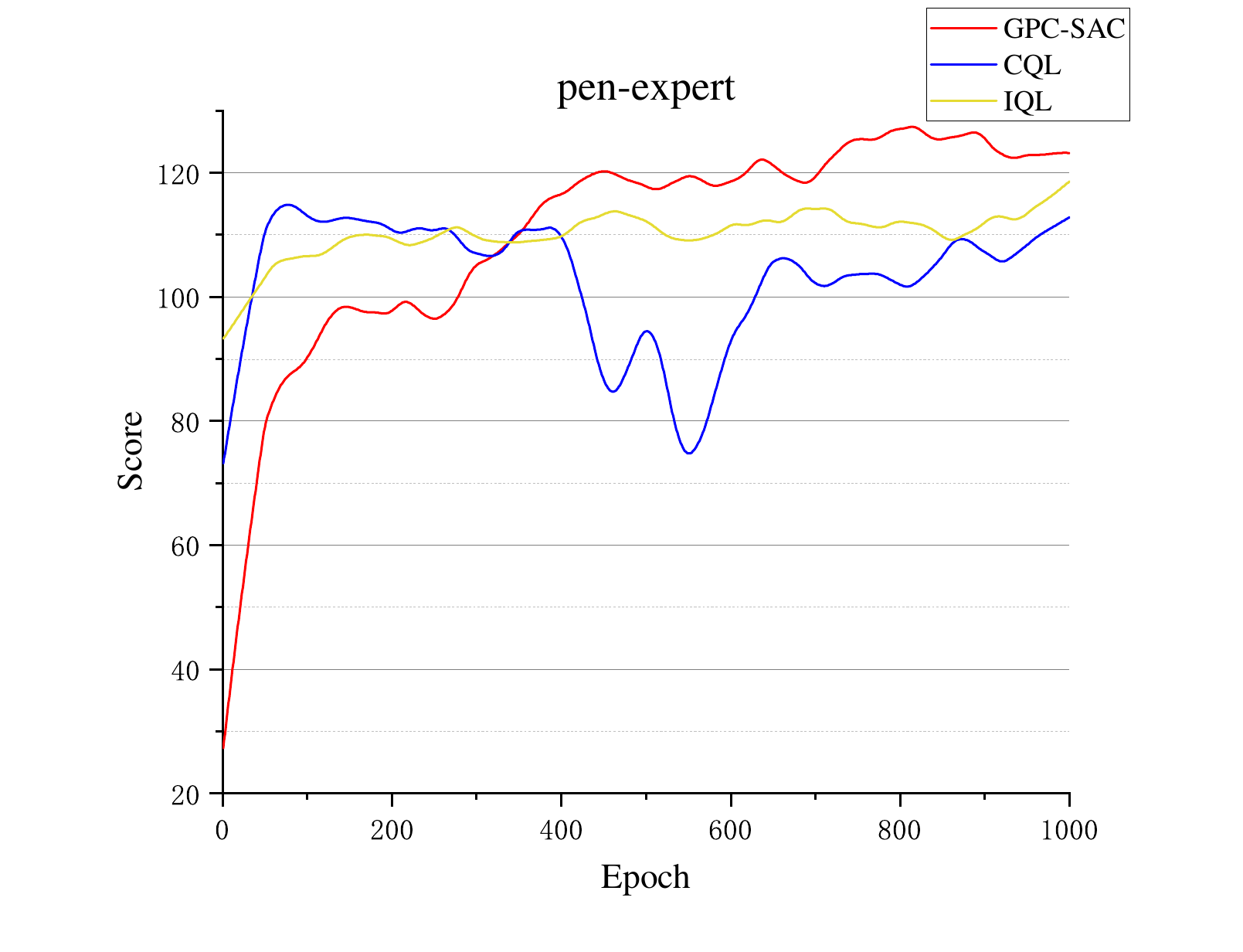}}
	\subfigure[hammer-expert]{
		\includegraphics[width=0.31\linewidth]{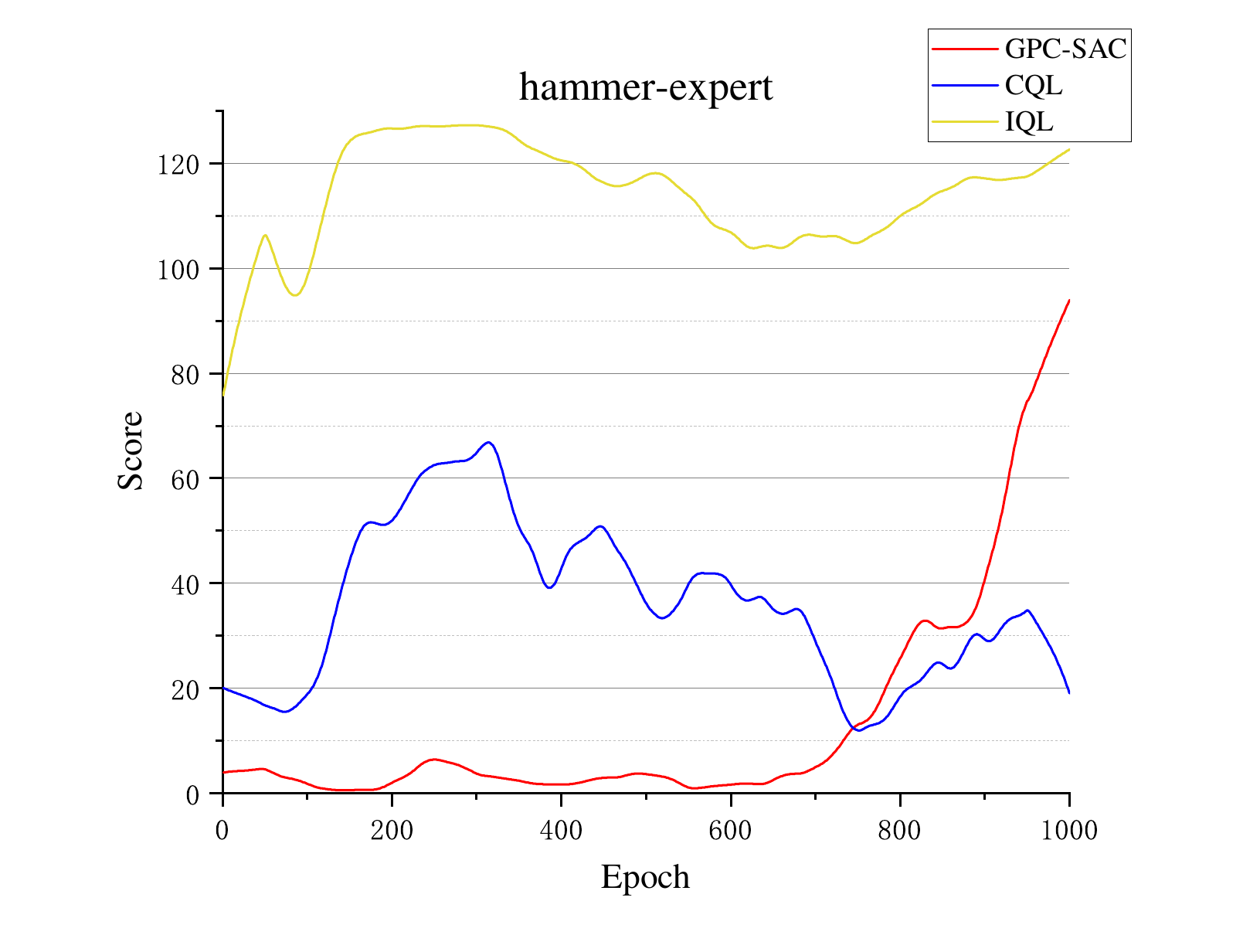}}
	\subfigure[door-expert]{
		\includegraphics[width=0.31\linewidth]{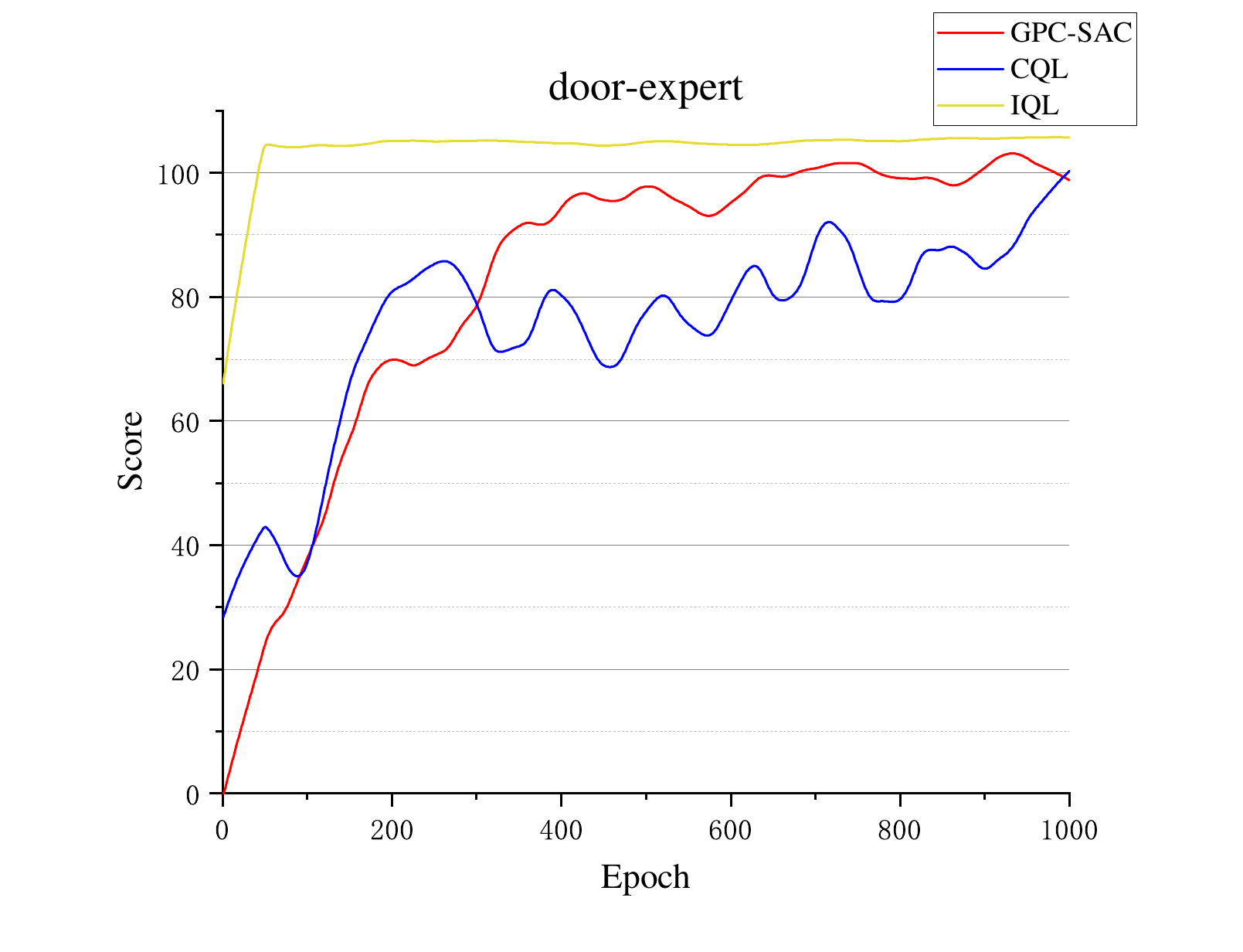}}
  \caption{adroit training curve}\
\end{figure*}

In each environment, the training curve of GPC-SAC is compared with IQL and CQL, as shown in Figure \ref{exp}.
In addition to the experimental results in Table \ref{mainres}, the training curves of GPC-SAC and various algorithms are compared. In halfcheetah-expert, both CQL and PBRL, which impose constraints on the Q-value, fail to find a good policy within 1500 epochs, both CQL and PBRL require 3000 epochs of training to find an optimal policy. Conversely, a perform-well policy can be found by GPC-SAC within 1500 epochs and as the training progresses it can be further improved. For the hopper-random environment where most algorithms fail to get good results, GPC-SAC not only demonstrates the best performance but also exhibits the best stability and training speed. Overall GPC-SAC has good stability in d4rl environments and is able to get the best scores within a short epoch, it can be said that GPC-SAC reaches the SOTA in d4rl environments.

\subsection{Experiment in other environment}
\subsubsection{Experiment in Maze2d}
\label{maze}
One advantage of GPC-SAC is its capacity to address complex problems in low-dimensional situations. To verify this advantage, experiments are conducted on Maze2D environments. Maze2D represents a maze navigation task designed to evaluate the efficacy of offline RL algorithms in connecting suboptimal trajectories to identify the shortest path to a target point. Maze2D comprises three tiers of mazes, from easy to challenging: umaze, medium, and large. Table \ref{score_ad} are the training results of GPC-SAC and some SOTA algorithms in Maze2D.
It is evident that GPC-SAC delivers remarkable performance across all Maze2D tasks, especially in maze2d-large. All other algorithms we know can not reach 65 scores. GPC-SAC achieves a score of 134.3 points, significantly surpassing current SOTA algorithms.

The shortcoming of count-based methods is that as the state-action space increases, the possible state-actions exponentially increase. This is a problem currently plaguing count-based class methods. One advantage of GPC-SAC over other count-based methods is its capability to treat state space and action space independently. Specifically, in scenarios with extra large action dimensions, GPC-SAC is not affected by it as GPC-SAC only requires states within the dataset. For the problem of large state dimension, as the amount of data in static datasets is limited, it will not affect GPC-SAC. For the problem of large action dimension, This problem is currently solved by reducing the number of action space partitions $k_{1}$. 
Adroit is a series of more complex environments compared to Gym, and the action dimensions of Adroit are all surpassing 20 dimensions. Therefore, there is a huge challenge for algorithms that use count-based methods.
To demonstrate the effectiveness of GPC-SAC in addressing high-dimensional problems, experiments are conducted in the Adroit environments. The experimental results are presented in Table \ref{score_ad}.

Even though count-based methods encounter challenges in solving high-dimensional problems, GPC-SAC also exhibits commendable performance in Adroit environments, especially in pen-expert and door-expert, GPC-SAC can perform well. 

\subsection{Hyperparameter}
\label{implement detail}
\begin{table}
\vspace{-2.0em}
\caption{Partial hyperparameters in GPC-SAC}
    \centering
    \begin{tabular}{lr}
        \toprule
        Hyperparameter   & value  \\
        \midrule
        discount         & 0.99     \\
        policy\_rl        & 3e-4   \\
        qf\_rl            & 3e-4     \\
        soft\_target\_tau  &  5e-3  \\
        train\_pre\_loop       & 1000    \\
        Q-network     & FC(256,256,256)    \\
        Optimizer     & Adam    \\
        activation\_function       & relu     \\
        $\beta$         &  1        \\
        $\beta_{next}$  & 0.1       \\
        \bottomrule
    \end{tabular}
    \label{parameter1}
\end{table}
\begin{table*}[!htp]
\vspace{-2.0em}
\caption{The number of the state space and action space has been partitioned in different D4RL environments. And the uncertainty constraint parameter $\alpha$. Here, maze-d is maze2d-diverse, maze-m is maze2d-medium and maze-l is maze2d-large.}
    \centering
    \begin{tabular}{ccccccc}
        \toprule
        Environment  & action\_dim & \textbf{$m$} & mean& (max+min)/2  & \textbf{$\alpha$}   &epoch\\
        \midrule
        hp-r           &3 & 15 &0.83 &1.35     & 1  &1000 \\
        hp-m-r     &3 & 15 &2.37 &2.57      & 3  &3000\\
        hp-m           &3 & 15 &3.11 &3.25       & 4  &3000  \\
        hp-m-e     &3 & 25 &3.36 &3.56      & 5  &3000\\
        hp-e           &3 & 25 &3.61 &3.76       & 5  &3000\\
        half-r     &6  & 5 &-0.29 &-0.41        & 0  &1000 \\
        half-m-r &6 & 5 &3.09 &2.16      & 2  &1000\\
        half-m      &6 & 5 &4.77 &2.75        & 3  &1000 \\
        half-m-e &6 & 5 &7.71 &5.42      & 6  &1000\\
        half-e       &6 & 7 &10.66 &5.42       & 6  &3000\\
        w-r         &6 & 6 &0.10 &0.20        & 2  &200 \\
        w-m-r    &6 & 6 &2.47 &1.92      & 3  &1000\\
        w-m          &6 & 6 &3.39 &2.96       & 4  &1000\\
        w-m-e  &6 & 6 &4.16 &2.96        & 5  &1000\\
        w-e         &6 & 6 &4.92 &3.27        & 5  &1000\\
        maze-d     & 2   & 30  & 0.08  &  0.50   & 0.2  &1000\\
        maze-m     & 2   &  30     & 0.02    & 0.50    &  0.2  &1000 \\
        maze-l     &2     & 30      &0.01     &   0.50    &0.2    & 1000 \\
        pen-e      &24     &1       &29.25    &27.38    &30   &1000\\
        hammer-e   &26   &1       &61.88   &55.7     &70    &1000\\
        door-e     &28    &1   &14.55  &9.75  &15   &1000\\
        \bottomrule
    \end{tabular}
    \label{parameter2}
\end{table*}

\begin{figure*}[!htp]
	\centering
 \vspace{-0.8em}   
\setlength{\abovecaptionskip}{0.cm}  
\setlength{\abovecaptionskip}{0.cm}  
	\label{beta}
        \subfigure{\includegraphics[width=0.48\linewidth]{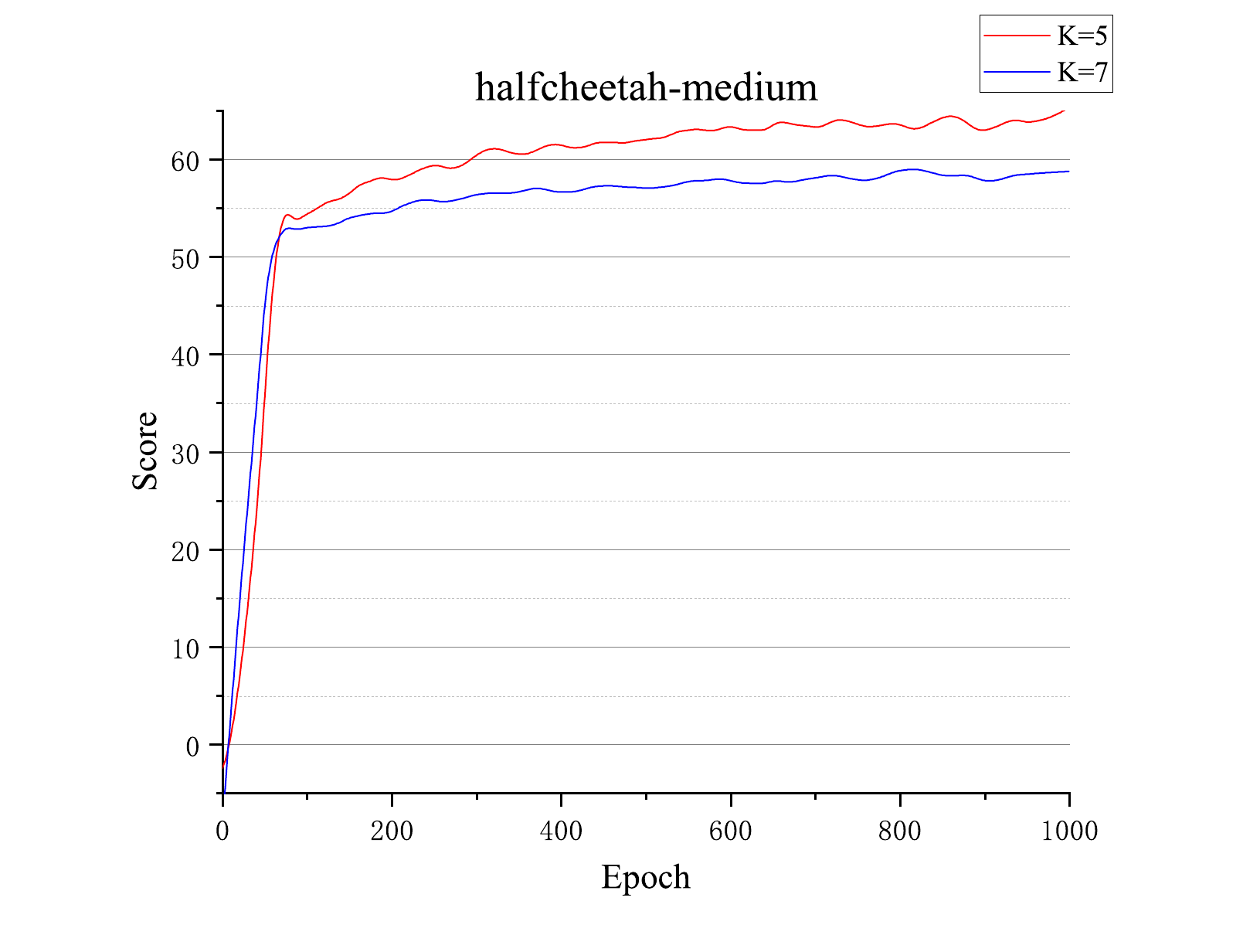}}
	\subfigure{\includegraphics[width=0.48\linewidth]{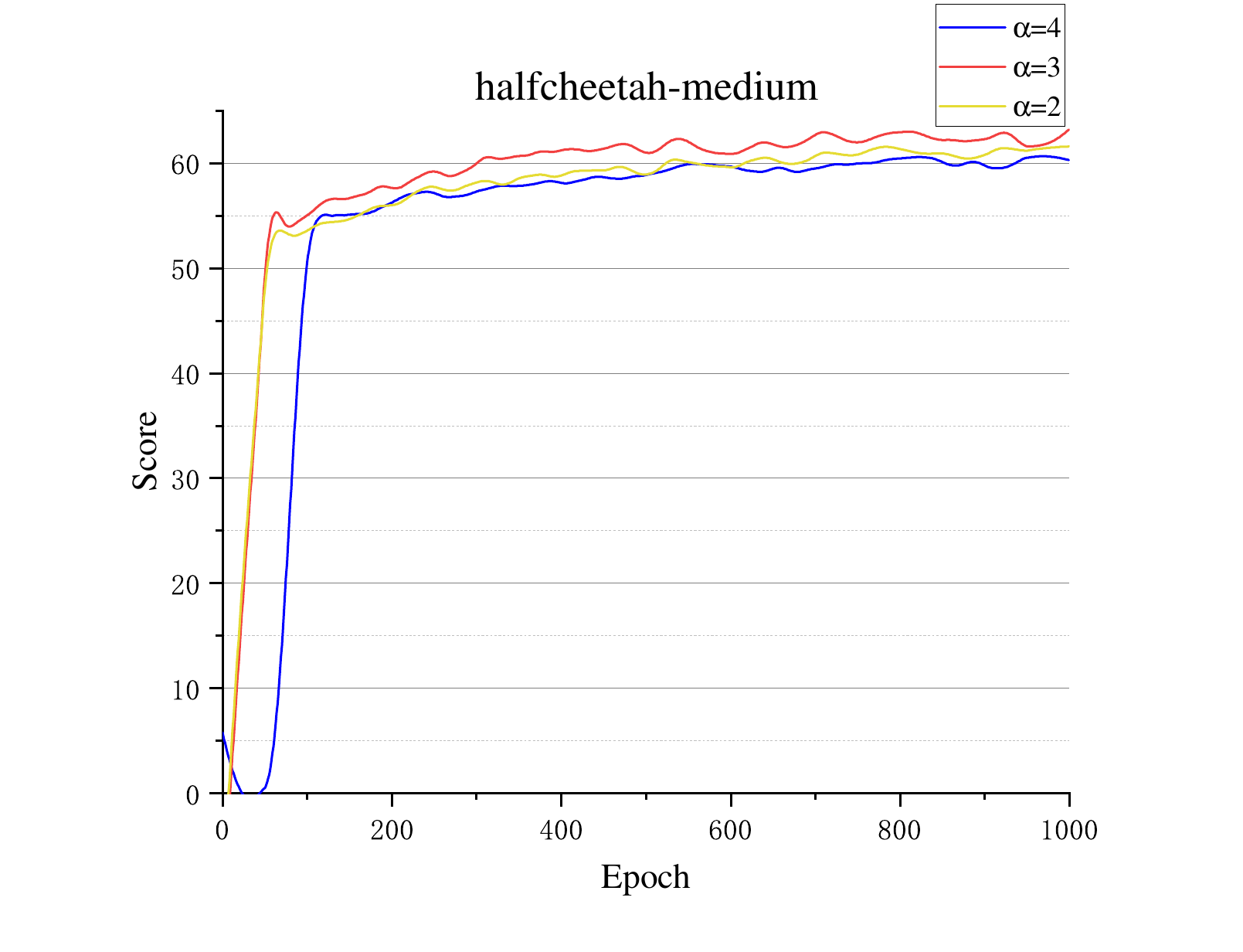}}
	\caption{The ablation on $m$ and $\alpha$}
	\label{beta and k}
\end{figure*}

\subsection{Hyperparameter setting}
For the hyperparameters, GPC-SAC is modified from SAC, and the selection of most hyperparameters in GPC-SAC follows that of SAC.
 In addition to the hyperparameters used in SAC, some additional hyperparameters are introduced in GPC-SAC: the number of partitions $\alpha$ for each dimension of state-action space, uncertainty constraint parameter $\alpha$ and uncertainty coefficient $\beta$ for $Q_{next}$. Table \ref{parameter1} shows the hyperparameters remain consistent across all environments.

$m$ and $\alpha$ are the key hyperparameters for GPC-SAC, Table \ref{parameter2} shows the hyperparameter combinations used in D4RL experiments. The values of $m$ and $\alpha$ are chosen based on the dimension of the state-action space. The selection of $m$ is crucial. When the selected $m$ is too large, it will slow down the training speed. While the $m$ is too small, it may affect the training results. The value of $\alpha$ is crucial too as it determines whether appropriate uncertainty constraints can be obtained. When selecting appropriate $\alpha$, the count-based constraint method can approximate epistemic uncertainty. 

For parameters $m$ and $\alpha$, in addition to the above analysis, reference formulas are also provided. Although using specific parameters in some environments may lead to better results, using the parameters provided can also achieve good results. For parameter $n$, selecting values within the range $12 \le \sqrt m {d_{action}} \le 16$ are recommended, where $d_{action}$ expressed the dimension of action space. 
For parameter $\alpha$, selecting values close to ${(min  + max)/2 }$ and the mean values of the dataset are recommended. 

The sensitivity of hyperparameter selection within a certain range is relatively low. In the experiment, when the value of $\alpha$ is close to the mean and ${(min  + max)/2 }$, good results can be obtained. Similarly, for $n$, when $n$ is selected that falls within the range provided, good results can also be obtained, which can be seen from Figure\ref{beta and k}.
In addition to hyperparameters $m$ and $\alpha$, we also conducted ablation experiments on other hyperparameters, and the specific experimental results can be seen in Appendices.

\subsection{Computational cost}
\begin{table}[h!]
\caption{The computational cost of some algorithms on the D4RL dataset. Runtime is the average of one epoch run in each environment. GPU Memory is the average of GPU Memory used in each environment.}
    \centering
    \begin{tabular}{ccc}
        \toprule
        Algorithm   & Runtime(s) & GPU Memory(GB) \\
        \midrule
        IQL       &13.6         & 2.0       \\
        SAC       & 42.1        & 1.3        \\
        CQL       & 58.8        & 1.4        \\
        PBRL      & 178.9         & 1.8        \\
        \textbf{GPC-SAC}   & 47.7      & 1.3        \\
        \bottomrule
    \end{tabular}
    \label{traintime}
\end{table}

Besides comparing performance, the computational time and computational space of GPC-SAC are compared with PBRL, SAC and CQL which constrain Q-values and IQL, another type of algorithms that learn policy through behavioral cloning. In the experiment, the number of ensembles used in PBRL is 10, which is consistent with the usage in the original paper.
The comparative experiments are run on a single NVIDIA GeForce RTX 3090, we compare the time and space of training for one epoch. 

From Table \ref{traintime}, it can be seen that IQL has the shortest training time because IQL updates policies based on behavior cloning, which reduces the number of operations compared to other algorithms based on SAC. However, due to the additional introduction of the V-value network, IQL requires a significant amount of memory usage. Comparing GPC-SAC with SAC, only a small amount of training time and operation space increased. Compared with PBRL using ensemble method quantization uncertainty and CQL using regularization constraints, GPC-SAC using GPC has shorter training time and lower computational space. This is because PBRL optimizes the network sequentially to obtain the best results; CQL requires multiple gradient computations brought by additional regularization. However, only a simple parallel operation of discretizing action space is needed by GPC-SAC, so it has less computational cost.

\section{Conclusions and Limitations}
In this paper, we proposed GPC, a novel method that utilizes pseudo-counting to quantify uncertainty in offline RL and use it to constrain the Q-value of OOD state-actions, resulting in an algorithm GPC-SAC. In theory, we proved that accurate uncertainty constraints can obtained by GPC under fewer conditions than other pseudo-count methods. The experimental results of GPC-SAC show that the performance of GPC-SAC is the best in most environments, especially in low-dimensional complex environments. On the other hand, the computational cost of GPC-SAC is lower than all other algorithms that constrain the Q-values. 

Like other count-based algorithms, the limitation of GPC-SAC is the possible state-actions increasing exponentially with dimensionality. For future work, we hope to study and completely solve this problem. 

\section{ Acknowledgements}
This work was supported by the National Natural Science Foundation of China [grant numbers: 12201656], Science and Technology Projects in Guangzhou [grant numbers: SL2024A04J01579] and Key Laboratory of Information Systems Engineering (CN).

\begin{appendix}
\setcounter{table}{0}   
\setcounter{figure}{0}
\setcounter{section}{0}
\setcounter{equation}{0}
\setcounter{lemma}{0}
\setcounter{theorem}{0}
\section{Algorithm Implementation}
\subsection{Implementation Details}
\label{Implementation Details}
In experiments, the handling of state space and action space is different. Since the selection of states is constrained to the existing states in the static dataset, the selection of states remains independent of the entire training process, so it does not need to gridding states in each training but directly store the grid-mapped values corresponding to each state. Therefore, this article only focuses on dealing with the actions used during each training session. 

Each dimension of the state space and action space are mapped separately by GPC. To prevent algorithm failure caused by the non-existent of Eq\eqref{gridOpertor} caused by the static dataset having only one possible value in a certain dimension, the following strategies are adopted: For state, whether to map that dimension before training is computed. This processing can also reduce the computational space used. For action, $10^{-6}+a_{max}-a_{min}$ is used to prevent algorithm failure. 

\subsection{Grid-Mapping Details}
\label{grid-map detal}
For the implementation of grid-mapping, 
Eq\eqref{gridOpertor} is proposed and $\alpha_{1},\alpha_{2} \ge1$ are choosen. Because during training, the integer corresponding to one state-action may become equivalent to the integer corresponding to another state-action.
For example, for a two-dimensional state-action space, using $m=2$, it can be seen from Eq\eqref{3} that 
\begin{equation}
    v (0,1) = v (1, - 1)=1
\end{equation}
This scenario means an OOD state-action corresponding to the vector (1,-1) is equivalent to another unrelated state-action in the dataset corresponding to vector (0,1). This relationship may lead to an overestimation of the Q-value of some OOD state-actions during training. Specifically, due to  (0,1) having been visited enough times, it results in smaller uncertainty constraints for the state-action (1,-1).  Consequently, when the state-action corresponding to (1,-1) is selected during training, which may cause an overestimation of its Q-value, potentially have a negative impact on the training outcomes.

Eq\eqref{gridOpertor} can consider OOD state-actions that are closer to those state-actions in the dataset to mitigate the impact of such errors. Using Eq\eqref{gridOpertor} can ensure the probability of the policy selecting an overestimated state-action is as small as possible. The detailed proof can be seen in \ref{gridproof}, and the following is an example to illustrate this method.

Choosing a two-dimensional space with $m=6$ and $\alpha=2$ as an example, where $s', a' \in \{0,1,2\}$ after gridding are the state-actions within the dataset. In this situation it has
\begin{equation}
\begin{array}{c}
\forall (s,a) , s,a \in Z  \\
v(s,a) = v(s-k',a+6k') , k' \in Z
\end{array}
\end{equation}
So, for these $s', a'\in \{0,1,2\}$, while $k'\neq0$
, it have $s-6k'\le-4 $ or $s-6k' \ge 6$. As the policy is learned from static datasets, the state-actions such as (2,6), (0,-4) that are far away from the static dataset will hardly be selected. The state-actions closer to the static dataset are more likely to be selected than those far away from static datasets, which will have the same properties as state-actions in the dataset. This method can ensure that $(s$,$a)$ or the corresponding $(s-k'$,$s+6k')$ has a sufficiently low probability of being selected by the policy, so this method can effectively reduce the impact of the error. The proof can be seen in \ref{gridproof}.

\setcounter{table}{0}   
\setcounter{figure}{0}
\setcounter{section}{0}
\setcounter{equation}{0}
\section{Proofs}
\label{proof and assumption}
\subsection{Count-based Uncertainty}
\label{unc}
In this section, the explanation of the rationality of count-based method in discrete space are supplemented.

\begin{lemma A.}[Hoeffding's inequality] 
\label{lemma1}
Let ${X_{1},...,X_{m}}$ be independent random variables, and ${\it X_{i}\in[a,b], i=1,...,m}$. The empirical mean of the random variable is expressed as:\[{\bar{X_{t}}} = \frac{1}{\tau}\sum\nolimits_{\tau  = 1}^{t} {{X_{\tau} }} \] 
Then Hoeffding's theorem states that:
\begin{equation}
\label{hoeffding}
    \forall u>0, P[E[\bar{X}_{t}]<X_{t}-u]\le e^{-2tu^2}
\end{equation}
\end{lemma A.}%
The proof can be found in ~\cite{hoeffding1994probability}. 
Under the condition of Lemma A.\ref{lemma1} it can obtained that $u(s,a)=\frac{lnT}{n(s,a)}$ is a suitable uncertainty constraint in discrete space.

Since $Q(s,a)$ is an independent random variable in offline RL, $Q(s,a)$ can be substituted into Lemma A.\ref{lemma1} to obtain:
\begin{align}
    P[Q(s,a)>(\bar{Q}(s,a)-u(s,a))]\ge e^{-2tu(s,a)^2}
\end{align}
Select a probability $p$ that allows, it can be obtained from the above equation that $e^{-2tu(s,a)^2}=p$, and then, from $e^{-2tu(s,a)^2}=p$ ,$u(s,a)=\sqrt{{-lnp}/{2n(s,a)}}$ can be obtained. In the application of offline RL, we hope that the estimation of the Q-value will become more accurate as the number of training times increases. Therefore, we hope that the value of $p$ gradually decreases as the training progresses. For example, when {\it T} is used to represent the current training epoch, ${\it p={1}/{T^2}}$meets the above requirements. Use this {\it p}, $\sqrt {{{\ln T}}/{{n(s,a)}}} $ can be obtained. 

\subsection{Uncertainty and $\Gamma^{\rm lcb}$}
\label{cnt_unc}
In this part, $\Gamma^{\rm lcb}$ is used to obtain an uncertainty constraint, and the rationality of using $\Gamma^{\rm lcb}$ as an uncertainty constraint can be demonstrated.

This problem is discussed in tabular cases and general cases.

For finite tabular MDPs, ${\Lambda _t} = \sum\nolimits_{i = 1}^m {\varphi (s_t^i,a_t^i)} \varphi {(s_t^i,a_t^i)^T} + \lambda  \cdot I$. In $\Lambda_t$, parameter $m$ represents all selected state-actions in the static dataset. In this situation, Lemma A.\ref{lemma3} can be get:
\begin{lemma A.}
\label{lemma3}
In tabular MDPs, the uncertainty $u(s,a)$ can be approximated by the counts of each state-action $(s,a)$ using the following method:
\begin{equation}
    u(s,a) = \frac{lnT}{{\sqrt {n(s,a) + \lambda}}}={\Gamma ^{\rm lcb}}(s,a)  
\end{equation}
\end{lemma A.}

\begin{proof}
  In tabular MDPs, define that  
\begin{align}
  \varphi ({s_i},{a_i}) = \left( \begin{array}{c}
0\\
 \vdots \\
{(\ln T)^{\frac{1}{4}}}\\
 \vdots \\
0
\end{array} \right) \in {R^d}
\end{align}
where $d=|S|\times|A|$, value 1 at $i$-th, all other positions in the vector are value 0. Which means that $\varphi({s_i},{a_i})$ is a count of $i$-th state-action $({s_i},{a_i})$. So it can compute that
\begin{equation}
\Lambda  = \left( {\begin{array}{*{20}{c}}
{\lambda  + {n_1}}& \cdots &0& \cdots &0\\
 \vdots & \ddots &{}&{}&{}\\
0&{}&{\lambda  + n_i}&{}&{}\\
 \vdots &{}&{}& \ddots &{}\\
0& \cdots &0& \cdots &{\lambda  + {n_d}}
\end{array}} \right) \in {R^{d\times d}}
\end{equation}
The diagonal of the matrix represents the count of all possible state-actions in the tabular MDP. Then it can compute that:

\begin{equation}
    {\Lambda ^{ - 1}} = \left( {\begin{array}{*{20}{c}}
{\frac{1}{{\lambda  + {n_1}}}}& \cdots &0& \cdots &0\\
 \vdots & \ddots &{}&{}&{}\\
0&{}&{\frac{1}{{\lambda  + {n_i}}}}&{}&{}\\
 \vdots &{}&{}& \ddots &{}\\
0& \cdots &0& \cdots &{\frac{1}{{\lambda  + {n_d}}}}
\end{array}} \right)
\end{equation}

\begin{equation}
\label{eq25}
    {\Gamma ^{\rm lcb}}({s_i},{a_i}) = \varphi ({s_i},{a_i})'{\Lambda ^{ - \frac{1}{2}}}\varphi ({s_i},{a_i}) = \sqrt{\frac{lnT}{{ \lambda  + {n_i}} }}
\end{equation}
$\sqrt{{lnT} / (\lambda  + {n_i})} $ is a common uncertainty constraint in tabular RL, which can be obtained by taking ${lnT}/{t}\ $at Corollary \ref{co1}.
By the above process, Eq\eqref{eq25} can be get. So, $\Gamma^{\rm lcb}$ is a reasonable uncertainty constraint in tabular RL. 
\end{proof}

For general situations, the epistemic uncertainty in a state-action $(s_t,a_t)$ can be considered as the discrepancy between the true Q-value $Q(s_t,a_t)$ and the estimated Q-value $Q^\pi(s_t,a_t)$, which is
\begin{equation}
\varepsilon  = {Q^\pi }({s_t},{a_t}) - Q({s_t},{a_t})
\end{equation}
Due to the assumption of linear MDP, $Q^{\pi}$ and $Q$ both have a linear relationship with $\varphi$, just like
\begin{equation}
{Q^\pi }({s_t},{a_t}) = {\hat \omega_t}'\varphi ({s_t},{a_t})\\,
Q({s_t},{a_t}) = {\omega_t}'\varphi ({s_t},{a_t})
\end{equation}
Where $\hat \omega_t$ is the estimated parameter and $\omega_t$ is the actual parameter.
Regarding the parameter $\hat \omega_t$ is used to estimate the Q-value, since $Q$ can be expressed as $r$ and $V$ by Bellman Equation, the value of $\hat \omega_t$ can be obtained by minimizing the empirical Mean-Square Bellman-Error (MSBE) and regularization term as follows:
\begin{align}
{\hat \omega _t} = \mathop {\min }\limits_{\omega_t  \in {R^d}} \left| {{{\sum\limits_{i = 1}^K {(r_t^i  + {{\hat V}_{t + 1}}(s_{t + 1}^i ) - \varphi (s_t^i ,a_t^i )'\omega )}^2 }} + \lambda \left\| \omega  \right\|_2^2} \right|
\end{align}
The explicit solution of $\hat \omega _t$ can be calculated as:
\begin{equation}
{\hat \omega_t} = \Lambda _t^{ - 1}\sum\limits_{i = 1}^m {\varphi (s_t^i,a_t^i)} ({V_{t + 1}}(s_{t + 1}^i) + r(s_t^i,a_t^i))
\end{equation}
so the epistemic uncertainty can be represented by $\omega_t$, that is
\begin{equation}
\label{31}
\varepsilon  = r(s_t^i,a_t^i) + {V_{t + 1}}(s_{t + 1}^i) - {\omega _t }'\varphi (s_t^i,a_t^i)
\end{equation}
In addition to the above analysis, Assumption A.\ref{as1} is needed.
\begin{assumption A.}
\label{as1}
The prior parameter $ \omega_{t}$ known from the dataset and epistemic uncertainty $\epsilon$ satisfies 
\begin{equation}
\omega_{t} \sim N(0,I\lambda ),\varepsilon \sim N(0,1)
\end{equation}
\end{assumption A.}

Under the condition of Assumption A.\ref{as1}, the following Lemma can be drawn:
\begin{lemma A.}
    Under Assumption A.\ref{as1}, $\Gamma ^{\rm lcb}$ can serve as an approximation to the actual uncertainty.
\end{lemma A.}
\begin{proof}
    $\Gamma ^{\rm lcb}$ can approximate real uncertainty, which means that 
\begin{equation}
    Var(Q^\pi(s_t^i,a_t^i)) \approx {\Gamma ^{\rm lcb}}{(s,a)^2}
\end{equation}
    where
    \[Var(Q^\pi(s_t^i,a_t^i)) \approx Var(\hat \varphi (s_t^i,a_t^i)'{\hat \omega_t })\]
    so the following equation needs to be proven 
\begin{equation}
    Var(\hat \varphi (s_t^i,a_t^i)'{\hat \omega _t}) \approx {\Gamma ^{\rm lcb}}{(s,a)^2}
\end{equation}
    So $\hat \omega _t$ needs to be obtained.
    From Eq\eqref{31} it can know that           
\begin{equation}    
    \begin{array}{l}
r(s_t^i,a_t^i) + V_{t+1}(s_{t + 1}^i)= \varepsilon  + {\omega_t }'\varphi (s_t^i,a_t^i)
\end{array} 
\end{equation}
Since $\varepsilon$ satisfies normal distribution, The following results can be obtained:
\begin{equation}
(r(s_t^i,a_t^i) + {V_{t + 1}}(s_{t + 1}^i))|(s_t^i,a_t^i) \sim N({{ \omega _t}}\varphi (s_t^i,a_t^i),1)
\end{equation}

From ${{\hat \omega_t }}'\varphi (s_t^i,a_t^i) = r(s_t^i,a_t^i) + V_{t+1}(s_{t + 1}^i)$ it can be known that:
\begin{equation}
\hat \omega_t \sim N({{ \omega_t }},1)
\end{equation}
When using $p’$ to denote the probability density function of the distribution. Using Bayes rule it has
\begin{equation}
\begin{array}{c}
p'( \omega_t |D) = \frac{{p'( \omega_t )p'(D| \omega_t )}}{{p'(D)}}
\log p'( \omega_t |D) \\
= \log p'( \omega_t ) + \log p'(D| \omega_t ) + d_1
\end{array}
\end{equation}
Here, $d_1$ is a constant. According to the probability density function of normal distribution, the following equations can be get:
\begin{equation}
\log p'( \omega_t ) =  - \frac{{{{\left\| { \omega_t } \right\|}^2}}}{2} + {d_2}
\end{equation}
\begin{equation}
    \begin{array}{c}
\log p'(D| \omega_t ) = \log \sum\limits_{i = 1}^m {p'((s_t^i,a_t^i,s_{t+1}^i)| \omega )} \\
 = \log \sum\limits_{i = 1}^m {p'(p(s_t^i)| \omega_t )}  = \log \sum\limits_{i = 1}^m {p'(\hat \omega_t |\omega_t )} \\
 = \frac{{\sum\limits_{i = 1}^m {{{\left\| { \omega_t '\varphi (s_t^i,a_t^i) - y_t^i} \right\|}^2}} }}{2}
\end{array}
\end{equation}
so 
\begin{equation}  
\begin{array}{c}
\log p'( \omega_{t} |D) =  - \frac{{{{\left\| { \omega_t } \right\|}^2}}}{2} - \frac{{\sum\limits_{i = 1}^m {{{\left\| { \omega_{t}'\varphi (s_t^i,a_t^i) - y_t^i} \right\|}^2}} }}{2} + b\\
 =   \frac{-{(\left\| { \omega_t } \right\| - \Lambda _t^{ - 1}\sum\limits_{i = 1}^m {\varphi (s_t^i,a_t^i)y_t^i} )'\Lambda _t^{ - 1}(\left\| {\hat \omega_t } \right\| - \Lambda _t^{ - 1}\sum\limits_{i = 1}^m {\varphi (s_t^i,a_t^i)y_t^i} )+2b}}{2}
\end{array}
\end{equation}
$d_1$ and $d$ are constants. From the above equation, The following can be obtained
\begin{equation}
{\hat \omega _t} = \omega_{t} |{D_{in}} \sim N(\Lambda _t^{ - 1}\sum\limits_{i = 1}^m {\varphi (s_t^i,a_t^i)} y_t^i,\Lambda _t^{ - 1})
\end{equation}

so for all $(s_t,a_t)$,
\[Var(Q^\pi(s_t^i,a_t^i)) \approx Var(\hat \varphi (s_t^i,a_t^i)'{\hat \omega_t }) \approx {\Gamma ^{\rm lcb}}{(s,a)^2}\]
\end{proof}

Although the above theorem requires that both $\hat \omega$ and $\varepsilon$ satisfy normal distribution, it can be inferred from the above proof process that this conclusion can extend to distributions with additivity. That is, Lemma A.\ref{lemma2} also applies to partial distributions.

\subsection{Continuity of $\Gamma ^{\rm lcb}$}
\label{apd_cont}
In this section, the proof of Lemma \ref{cont} is also added, which is why $\Gamma ^{\rm lcb}$ is considered a continuous function.

To illustrate the continuity of $\Gamma ^{\rm lcb}$,  $\Gamma ^{\rm lcb}$ should directly be calculated. For this, first represent $\Lambda$ in the following form
\begin{equation}
    \Lambda  = \left( {\begin{array}{*{20}{c}}
{\lambda {\rm{ + }}{{\rm{c}}_{11}}}&{...}&{{{\rm{c}}_{i1}}}& \cdots &{{{\rm{c}}_{d1}}}\\
 \vdots & \ddots &{}&{}& \vdots \\
{{{\rm{c}}_{i1}}}&{}&{\lambda {\rm{ + }}{{\rm{c}}_{ii}}}&{}&{{{\rm{c}}_{di}}}\\
 \vdots &{}&{}& \ddots & \vdots \\
{{{\rm{c}}_{d1}}}& \cdots &{{{\rm{c}}_{di}}}& \cdots &{\lambda {\rm{ + }}{{\rm{c}}_{dd}}}
\end{array}} \right)
\end{equation}
where ${c_{ij}} = \sum\nolimits_{l = 1}^m {{\varphi _l}(s_t^i,a_t^i){\varphi _l}(s_t^j,a_t^j)}$.
Since $\Lambda$ is an invertible matrix, the inverse of matrix $\Lambda$ can be directly calculated as
\begin{equation}
\label{Lamdba}
    {\Lambda ^{ - 1}} = \left( {\begin{array}{*{20}{c}}
{\frac{{{a_{11}}}}{{{b_{11}}}}}&{...}&{\frac{{{a_{i1}}}}{{{b_{i1}}}}}& \cdots &{\frac{{{a_{d1}}}}{{{b_{d1}}}}}\\
 \vdots & \ddots &{}&{}& \vdots \\
{\frac{{{a_{i1}}}}{{{b_{i1}}}}}&{}&{\frac{{{a_{ii}}}}{{{b_{ii}}}}}&{}&{\frac{{{a_{di}}}}{{{b_{di}}}}}\\
 \vdots &{}&{}& \ddots & \vdots \\
{\frac{{{a_{d1}}}}{{{b_{d1}}}}}& \cdots &{\frac{{{a_{di}}}}{{{b_{di}}}}}& \cdots &{\frac{{{a_{dd}}}}{{{b_{dd}}}}}
\end{array}} \right)
\end{equation}
where $a_{ij}$,$b_{ij}$ are polynomials related to $ \varphi (s_t^1,a_t^1)$, $\varphi (s_t^2,a_t^2)$, ..., $\varphi (s_t^m,a_t^m)$. Then, represent 
${(\varphi_1(s_t,a_t) \cdots \varphi_i(s_t,a_t) \cdots \varphi_d(s_t,a_t))^{'}}$
as $\varphi $. 
$\Gamma ^{\rm lcb}$ can be represented as
\begin{equation}
\label{eqlcb}
{\Gamma ^{\rm lcb}} = \varphi '{\Lambda ^{ - 1}}\varphi 
\end{equation}
Plugging Eq\eqref{Lamdba} into Eq\eqref{eqlcb} get
\begin{equation}
    {\Gamma ^{\rm lcb}} = \sum\nolimits_{i = 1}^d {\sum\nolimits_{j = i}^d {\sum\limits_{l = 1}^n {\frac{{{a_{li}}}}{{{b_{li}}}}} {\varphi _i}({s_t},{a_t})} {\varphi _j}({s_t},{a_t})} 
\end{equation}
From this equation, for kernel functions with continuous components, such as Gaussian kernel and other common kernel functions. ${\varphi _i}({s_t},{a_t})$ is continuous, and since ${a_{li}}/{b_{li}}$ is reversible, so ${b_{li}} \neq 0$, and the composition of continuous functions is also a continuous function, then ${a_{li}}/{b_{li}}$ is also a continuous function. And ${\Gamma ^{\rm lcb}}$ is the composite of ${\varphi _i}({s_t},{a_t})$ and ${a_{li}}/{b_{li}}$, therefore it is also a continuous function.

Therefore, based on the above proof, the following Theorem can be obtained, from which it can be known that Lemma \ref{cont} holds.
\begin{theorem A.}
    When the components of each dimension of kernel function $\varphi :S \times A \to {R^d}$ are continuous, ${\Gamma ^{\rm lcb}}$ is also continuous.
\end{theorem A.}

\subsection{Rational of GPC}
\label{gridproof}
\begin{theorem A.}
 Set $p(s,a)$ be the probability $(s,a)$ was selected in a discrete state-action space. $\forall \epsilon>0$, $p'(s_i,a_j)>\epsilon $ it has the probability of selecting a state-action with the same corresponding integer as $p'(s_i,a_j)<\epsilon$. 
\end{theorem A.}

\begin{proof}
We prove this conclusion with $k_1,k_2=\kappa \in Z$ in a two-dimensional state-action space, and for high-dimensional spaces, it can be obtained by analogy.
Since the states used in the training are only the existing states in the dataset, we can know 
\begin{align}
    s \in \{ {s_1}...{s_m}\} ,{s_i} \in Z
\end{align}
Therefore, take any possible $(s,a)$, there are only different possible outcomes less than $m$ satisfy $v (s,a)=v (s',a')$
so, $\forall s_i,a_i,p'(s_i,a_i)>\epsilon$, The following needs to be proven  
\begin{align}
\label{46}
\sum\limits_{j={1}}^{{m'}} {p({s_j},{a_j})} <\epsilon 
\end{align}
Where $({s_j},{a_j})=({s_i}-k,{a_i}+\kappa k)$, $m'< m$.
Due to that state-actions are obtained by agent learning from a static dataset, the lower the selection rate of actions that deviate from existing actions in the dataset, it can be written as:
\begin{equation}
\begin{array}{c}
\forall {s_i} \in \{ {s_1}...{s_k}\} ,\exists {a_{i1}}<{a_{i2}},\forall a<{a_{i1}},a> {a_{i2}}  \\
 p({s_i},a)<\frac{1}{{m'}}\epsilon 
\end{array}
\end{equation}

Denote ${a_{\min }}=\mathop {\min }\limits_{i = 1,...,m'} {a_{i1}}$
, ${a_{\max }}=\mathop {\max }\limits_{i = 1,...,m'} {a_{i2}}$, so when choose $k>a_{max}-a_{min}$ it can plugged into Eq\eqref{46} to get
\begin{equation}
\label{eq47}
\sum\limits_{j={1}}^{{m'}} {p({s_j},{a_j})}  \le m'\mathop {\max }\limits_{j = 1,...,m'} (p({s_j},{a_j})) \le \epsilon 
\end{equation}
By Eq\eqref{eq47}, the theorem is proofed.
\end{proof}

\subsection{Policy optimization guarantee}
\label{p_improve}
In this section, GPC-SAC can ensure the effectiveness of the policy will not decrease after each training in any environment is proven.
The proof section refers to SAC. To prove the conclusion, some relevant definitions of the section are provided.
Likes SAC, the relationship between V and Q in GPC-SAC is that 
\begin{equation}
\label{50}
V({s_t}) = {E_{{a_t}\sim \pi }}[Q({s_t},{a_t}) - \log \pi ({a_t}|{s_t})]   
\end{equation}
$Q({s_t},{a_t})$ is updated by
\begin{equation}
\label{51}
\begin{aligned}
Q({s_t},{a_t}) &= r({s_t},{a_t}) + E_{s_{t + 1}\sim p,a_{t + 1}\sim \pi}[{Q^{{\pi _{new}}}}({s_{t + 1}},{a_{t + 1}})] \\&- u({s_t},{a_t}) 
\end{aligned}
\end{equation}
For the policy $\pi_{new}$ obtained from each update, it can be represented as
\begin{equation}
\label{disbu}
{\pi _{\rm new}} = \arg \mathop {\min }\limits_{} {D_{KL}}(\pi '(.|{s_t})||\frac{{\exp ({Q^{{\pi _{\rm old}}}}(s_t, \cdot ))}}{{{Z^{{\pi _{\rm old}}}}({s_t})}})
\end{equation}
In this equation, $\pi_{\rm old}$ is the current policy, ${{{Z^{{\pi _{\rm old}}}}({s_t})}}$ is used to standardize the entire distribution. Although ${{{Z^{{\pi _{\rm old}}}}({s_t})}}$ is generally difficult to handle, it does not contribute to the gradient of the new policy $\pi _{\rm new}$, so it can be ignored.
After representing $\pi_{\rm new}$ as the above form, the following theorem can be obtained.
\begin{theorem A.}
\label{policy_improve}
Let $\pi_{\rm old}$ be the policy obtained from the current epoch training, $\pi_{\rm new}$ is the new policy optimized by Eq(\ref{disbu}), then it have ${Q^{{\pi _{\rm new}}}}({s_t},{a_t}) \ge {Q^{{\pi _{\rm old}}}}({s_t},{a_t})$ for all $({s_t},{a_t}) \in S \times A $. Here $ \left| A \right| < \infty $.
\end{theorem A.}

\begin{proof}
In Eq(\ref{disbu})
\begin{equation}
\begin{aligned}
\frac{{\exp ({Q^{{\pi _{\rm old}}}}({s_t}, \cdot ))}}{{{Z^{{\pi _{\rm old}}}}({s_t})}} = \frac{{\exp ({Q^{{\pi _{\rm old}}}}({s_t}, \cdot ))}}{{\exp (\log {Z^{{\pi _{\rm old}}}}({s_t}))}} 
\\= \exp ({Q^{{\pi _{\rm old}}}}({s_t}, \cdot )) - \log {Z^{{\pi _{\rm old}}}}({s_t}))    
\end{aligned}
\end{equation}

so 
\begin{equation}
\begin{aligned}
{\pi _{\rm new}}(.|{s_t}) &= \arg \min {D_{KL}}(\pi '(.|{s_t})||\exp ({Q^{{\pi _{\rm old}}}}({s_t}, \cdot )) \\
&- \log {Z^{{\pi _{\rm old}}}}({s_t}))   
\end{aligned}
\end{equation}    

Because expanding the KL divergence can get
\begin{equation}
\begin{aligned}
\begin{array}{l}
KL(\pi ({a_t}|{s_t})||\exp ({Q^\pi }({s_t}, \cdot )) - \log {Z^\pi }({s_t})) 
\\= {E_{{a_t}\sim \pi }}[\log \pi ({a_t}|{s_t}) - {Q^\pi }({a_t},{s_t}) + \log {Z^\pi }({s_t})]
\end{array}
\end{aligned}
\end{equation}

And since for $\pi _{new}$, even in the worst case, when choose $\pi _{new}=\pi _{old}$. 
So it has
\begin{equation}
\begin{aligned}
{E_{{a_t}\sim{\pi _{\rm new}}}}[\log {\pi _{\rm new}}({a_t}|{s_t}) - {Q^{{\pi _{\rm old}}}}({a_t},{s_t}) + \log {Z^{{\pi _{\rm old}}}}({s_t})] \\
\le {E_{{a_t}\sim{\pi _{\rm old}}}}[\log {\pi _{\rm old}}({a_t}|{s_t}) - {Q^{{\pi _{\rm old}}}}({a_t},{s_t}) + \log {Z^{{\pi _{\rm old}}}}({s_t})]   
\end{aligned}
\end{equation}

Since $Z$ only depends on the state, the above equation can be simplified as
\begin{equation}
\begin{aligned}
{E_{{a_t}\sim{\pi _{\rm new}}}}[\log {\pi _{\rm new}}({a_t}|{s_t}) - {Q^{{\pi _{\rm old}}}}({a_t},{s_t})] \\
\le {E_{{a_t}\sim{\pi _{old}}}}[\log {\pi _{\rm old}}({a_t}|{s_t}) - {Q^{{\pi _{\rm old}}}}({a_t},{s_t})]    
\end{aligned}
\end{equation}
Substituting Eq(\ref{50}) into the above equation it can be get that
\begin{equation}
{E_{{a_t}\sim{\pi _{\rm new}}}}[\log {\pi _{\rm new}}({a_t}|{s_t}) - {Q^{{\pi _{\rm old}}}}({s_t},{a_t})] \le {V^{{\pi _{\rm old}}}}({s_t})    
\end{equation}

From Bellman equation Eq(\ref{51}) and the above inequality, it can be known that
\begin{equation}
\label{59}
\begin{aligned}
\begin{array}{l}
{Q^{{\pi _{\rm old}}}}({s_t},{a_t}) = r({s_t},{a_t}) + \gamma {E_{{s_{t + 1}}\sim p}}[{V^{{\pi _{\rm old}}}}({s_{t + 1}})]\\
 \le r({s_t},{a_t}) + \gamma {E_{{s_{t + 1}}\sim p{a_{t + 1}}\sim{\pi _{\rm new}}}}[{Q^{{\pi _{\rm old}}}}({s_{t + 1}},{a_{t + 1}}) \\- \log {\pi _{\rm new}}({a_{t + 1}}|{s_{t + 1}})- {u^{{\pi _{\rm old}}}}({s_n},{a_n})]
 \end{array}
\end{aligned}
\end{equation}
On the other hand
\begin{equation}
\label{60}
   \begin{aligned}
       \begin{array}{l}
{Q^{{\pi _{\rm new}}}}({s_t},{a_t}) = r({s_t},{a_t})\\
 + \gamma E_{{s_{t + 1}}\sim p{a_{t + 1}}\sim {\pi _{\rm new}}}{Q^{{\pi _{\rm new}}}}({s_{t + 1}},{a_{t + 1}})\\
 - \log {\pi _{\rm new}}({a_{t + 1}}|{s_{t + 1}})] - {u^{{\pi _{\rm new}}}}({s_n},{a_n})
\end{array}
   \end{aligned} 
\end{equation}
So consider Eq(\ref{59}) and Eq(\ref{60}). The problem becomes the relationship between
\begin{equation}
E_{{s_{t + 1}\sim p},{a_{t + 1}}\sim{\pi _{\rm new}}}[{Q^{{\pi _{\rm old}}}}({s_{t + 1}},{a_{t + 1}})]
\end{equation}
and
\begin{equation}
    E_{{s_{t + 1}\sim p},{a_{t + 1}}\sim{\pi _{\rm new}}}[{Q^{{\pi _{\rm new}}}}({s_{t + 1}},{a_{t + 1}})]
\end{equation}
Since the MDP this parer is considering is finite, the above equation can be iterated to obtain:
\begin{equation}
\label{61}
\begin{aligned}
\begin{array}{l}
E_{{s_{t + 1}\sim p},{a_{t + 1}}\sim{\pi _{\rm new}}}[r({s_{t + 1}}) - {u^{{\pi _{\rm new}}}}({s_{t + 1}},{a_{t + 1}}) + ...\\
+[E_{{s_n}\sim p,a_n\sim \pi _{\rm new}}[r({s_n}) - {u^{{\pi _{\rm new}}}}({s_n},{a_n})...]]]   
\end{array}
\end{aligned}
\end{equation}
\begin{equation}
\label{62}
\begin{aligned}
\begin{array}{l}
E_{{s_{t + 1}\sim p},{a_{t + 1}}\sim{\pi _{\rm new}}}[r({s_{t + 1}}) - {u^{{\pi _{old}}}}({s_{t + 1}},{a_{t + 1}}) + ...\\
+[E_{{s_n}\sim p,{a_n}\sim{\pi _{\rm new}}}[r({s_n}) - {u^{{\pi _{\rm old}}}}({s_n},{a_n})...]]]
\end{array}
\end{aligned}
\end{equation}
Because for all $s$,$a$ it have ${n_{\rm new}}(s,a) \ge {n_{\rm old}}(s,a)$ so ${u^{{\pi _{\rm new}}}}({s},{a}) \ge {u^{{\pi _{\rm old}}}}({s},{a})$.
By Eq(\ref{60}), Eq(\ref{61}) and Eq(\ref{62}) it can known that
\begin{equation}
\label{63}
\begin{aligned}
\begin{array}{l}
r({s_t},{a_t}) + \gamma E_{{s_{t + 1}\sim p},{a_{t + 1}}\sim{\pi _{\rm new}}}[{Q^{{\pi _{\rm old}}}}({s_{t + 1}},{a_{t + 1}}) 
\\- \log {\pi _{\rm new}}({a_{t + 1}}|{s_{t + 1}})] - {u^{{\pi _{\rm old}}}}({s_t},{a_t})
\\\le r({s_t},{a_t}) + \gamma E_{{s_{t + 1}\sim p},{a_{t + 1}}\sim{\pi _{\rm new}}}[{Q^{{\pi _{\rm new}}}}({s_{t + 1}},{a_{t + 1}}) 
 \\- \log {\pi _{\rm new}}({a_{t + 1}}|{s_{t + 1}})] - {u^{{\pi _{\rm old}}}}({s_t},{a_t})
 \\= r({s_t},{a_t}) + \gamma E_{{s_{t + 1}\sim p}}[{V^{{\pi _{\rm new}}}}({s_{t + 1}})] - {u^{{\pi _{\rm new}}}}({s_t},{a_t}) 
 \\= {Q^{{\pi _{\rm new}}}}({s_t},{a_t})
 \end{array}
 \end{aligned}
\end{equation}
Combining Eq(\ref{59}) and Eq(\ref{63}) it can get that 
\[{Q^{{\pi _{\rm old}}}}({s_t},{a_t}) \le {Q^{{\pi _{\rm new}}}}({s_t},{a_t})\]
So the effectiveness of the policy will not decrease in GPC-SAC.
\end{proof}

\setcounter{table}{0}   
\setcounter{figure}{0}
\setcounter{equation}{0}
\section{Experiment}

\begin{figure}[!h]
\includegraphics[width=1\linewidth]{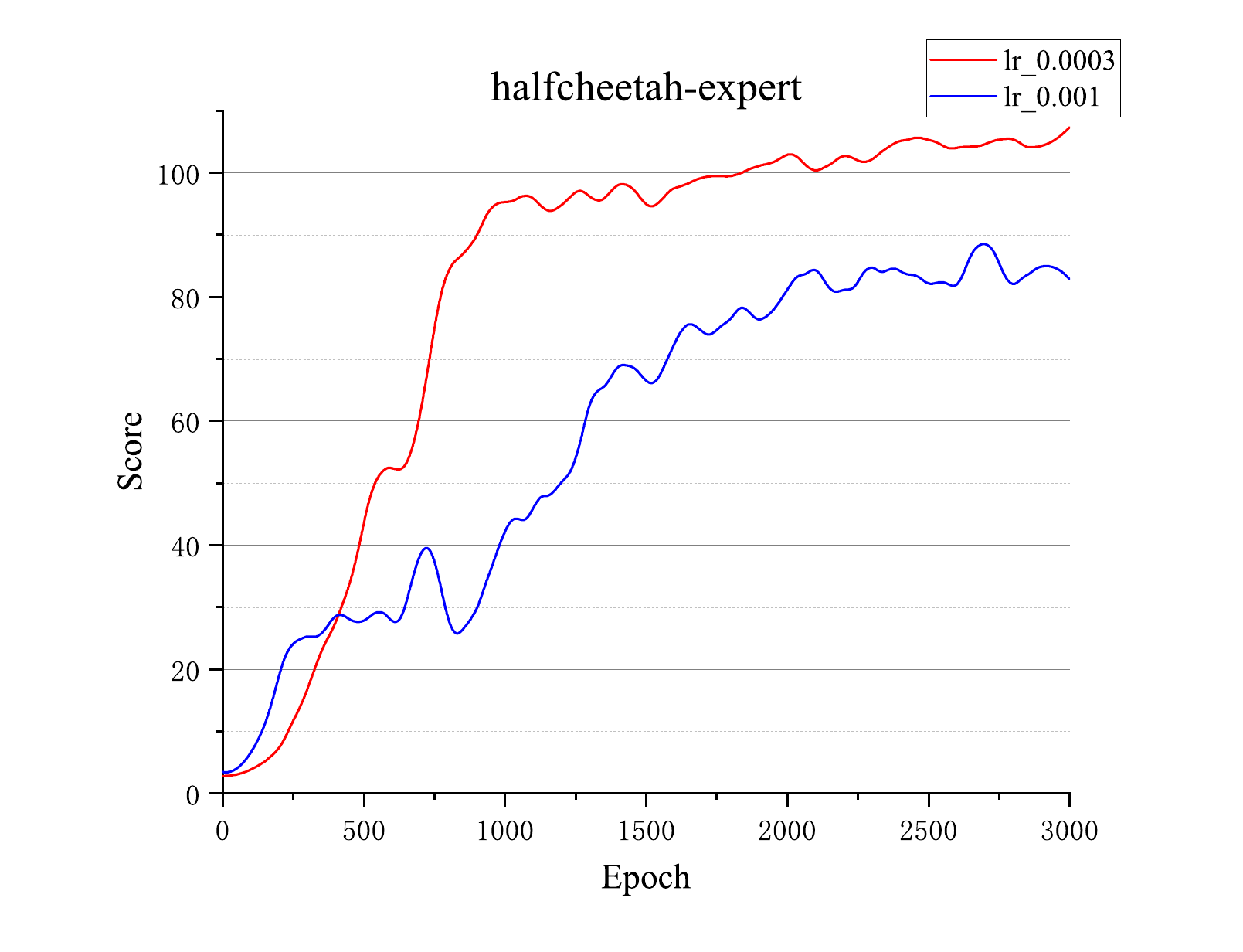}
\caption{The ablation on the learning rate}
\label{rl_rate}
\end{figure}
\begin{figure}
\includegraphics[width=1\linewidth]{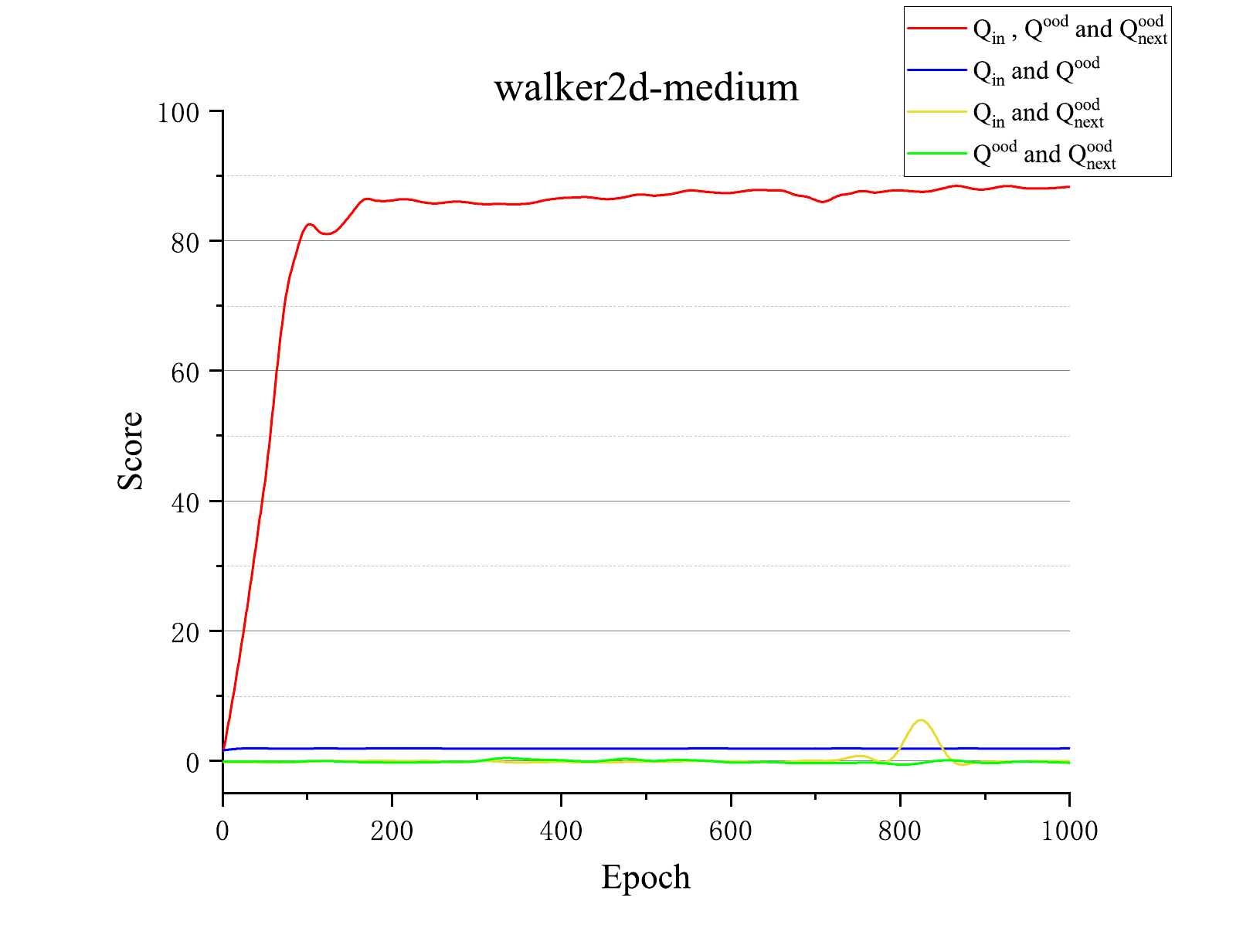}
\caption{The ablation on the Q-function update}
\label{Q-fuc}
\end{figure}

\begin{figure*}[!h]
	\centering
\setlength{\abovecaptionskip}{0.cm}  
\setlength{\abovecaptionskip}{0.cm} 
	\subfigure{
		\includegraphics[width=0.48\linewidth]{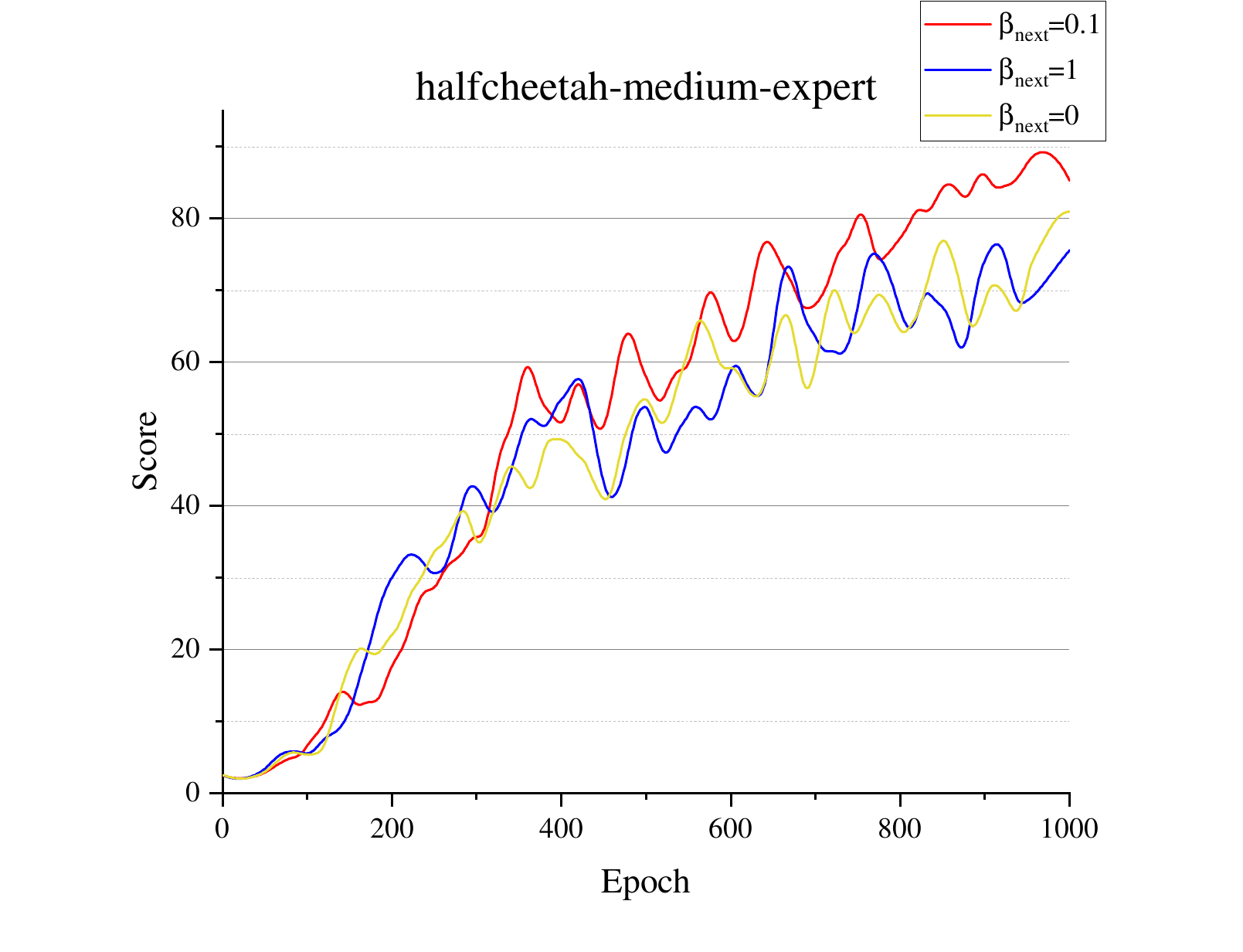}}
	\subfigure{
		\includegraphics[width=0.48\linewidth]{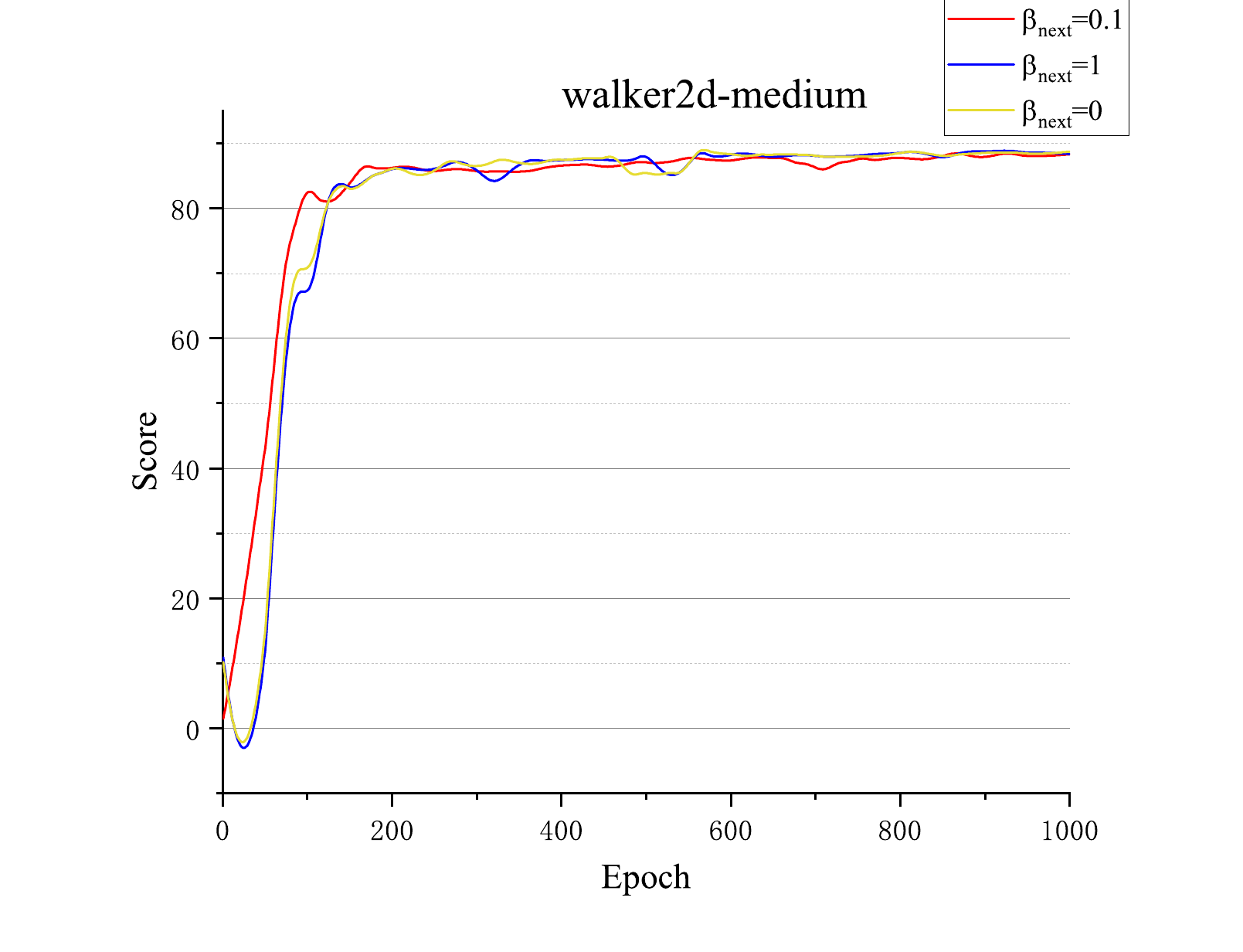}}
	\caption{The ablation on $\beta_{ood}$}
 \label{inood}
\end{figure*}

To replicate our results, the specific implementation of GPC-SAC can be found at \href{https://github.com/lastTarnished/GPC-SAC}{https://github.com/lastTarnished/GPC-SAC}.
We conducted experiments on the latest v2 dataset of D4RL, which can be publicly used at \href{https://github.com/Farama-Foundation/D4RL}{https://github.com/Farama-Foundation/D4RL}. The following are the implementations of the comparative algorithms in our experiment. (1)For CQL, we referred to \href{https://github.com/aviralkumar2907/CQL/tree/master}{https://github.com/aviralkumar2907/CQL/tree/master} which is the official implementation of CQL. (2)For the implementation of IQL, we refer to \href{https://github.com/ikostrikov/implicit\_q\_learning}{https://github.com/ikostrikov/implicit\_q\_learning} this code has simplified the implementation of IQL compared to the official code, resulting in a slight decrease in performance in a few environments. (3)For the implementation of PBRL, we refer to \href{https://github.com/Baichenjia/PBRL}{https://github.com/Baichenjia/PBRL} and no changes were made to the core code.

\subsection{Ablation Experiment}
\label{ablation}
\textbf{Learning rate}: The implementation of CQL in SAC framework and SAC itself use two different sets of Q-function learning rates, policy function learning rate and soft update rate. Two different parameters are compared to choose the appropriate one. The experimental results are in Figure \ref{rl_rate}.

For some parameters, ablation experiments are conducted to demonstrate the effectiveness of these parameters. The setting of ablation experiments is similar to that of the comparative experiment. The remaining parameters are kept unchanged and the parameters that changed are parameters need to be compared. The following are the settings and results of the ablation experiment.

\textbf{Application of OOD sample}: $Q(s_{\rm next}, a_{\rm next}^{\rm ood})$, $Q(s_{\rm in},a_{\rm in}^{\rm ood})$, and $Q (s_{\rm in}, a_{\rm in})$ are used to update the Q-function together in the update. Theoretically, the use of each type of Q-value can be divided into seven types. The following four situations are considered. using three parts at the same time, using $Q (s_{\rm in},a_{\rm in})$, $ Q (s_{\rm in},a_{\rm in}^{\rm ood})$, using $Q (s_{\rm in},a_{\rm in}^{\rm ood})$, $ Q (s_{\rm next},a_{\rm next}^{\rm ood})$, and using $Q (s_{\rm next},a_{\rm next}^{\rm ood})$, $ Q (s_{\rm in},a_{\rm in})$ . The experimental results are in Figure \ref{Q-fuc}.
No ablation experiments were conducted for the case of using only one type of Q, but the case of using only one type of Q can be obtained by analogy from the above results.

\textbf{Uncertainty parameter for $Q(s_{\rm next},a_{\rm next}^{\rm ood})$}: The difference between $s_{\rm next}$ and $s_{\rm in}$ is that $s_{\rm next}$ does not have known actions in the dataset to guide the agent's actions. When using large uncertainty constraints on $(s_{\rm in},a_{\rm in}^{\rm ood})$, the agent can still select $(s_{\rm in},a_{\rm in})$in the static dataset, ensuring that the value of $s_{\rm in}$ is not excessively underestimated. However, if the same uncertainty constraint is chosen for $(s_{\rm next},a_{\rm next}^{\rm ood})$, it may lead to an excessive underestimation of the value of $s_{\rm next}$ and $a_{\rm next}^{\rm ood}$, causing the agent misses some states and actions with high potentially, resulting in a decrease in algorithm performance. Therefore, ablation experiments are conducted on parameter $\beta_{ood}$, selecting $\beta_{\rm ood}$ from \{0, 0.1, 1\}, which respectively indicate not imposing constraints to $(s_{\rm next},a_{\rm next}^{\rm ood})$ at all, imposing smaller constraints to $(s_{\rm next},a_{\rm next}^{\rm ood})$, and imposing equal constraints to $(s_{\rm next},a_{\rm next}^{\rm ood})$ and $(s_{\rm in},a_{\rm in}^{\rm ood})$. The results of the experiment are shown in Figure \ref{inood}.

\end{appendix}
\bibliographystyle{abbrv}
\bibliography{main}

\begin{thebibliography}{10}

\bibitem{abbasi2011improved}
Y.~Abbasi-Yadkori, D.~P{\'a}l, and C.~Szepesv{\'a}ri.
\newblock Improved algorithms for linear stochastic bandits.
\newblock {\em Advances in neural information processing systems}, 24, 2011.

\bibitem{agarwal2020optimistic}
R.~Agarwal, D.~Schuurmans, and M.~Norouzi.
\newblock An optimistic perspective on offline reinforcement learning.
\newblock In {\em International Conference on Machine Learning}, pages 104--114. PMLR, 2020.

\bibitem{an2021uncertainty}
G.~An, S.~Moon, J.-H. Kim, and H.~O. Song.
\newblock Uncertainty-based offline reinforcement learning with diversified q-ensemble.
\newblock {\em Advances in neural information processing systems}, 34:7436--7447, 2021.

\bibitem{bai2022pessimistic}
C.~Bai, L.~Wang, Z.~Yang, Z.~Deng, A.~Garg, P.~Liu, and Z.~Wang.
\newblock Pessimistic bootstrapping for uncertainty-driven offline reinforcement learning.
\newblock {\em arXiv preprint arXiv:2202.11566}, 2022.

\bibitem{bellemare2016unifying}
M.~G. Bellemare, S.~Srinivasan, G.~Ostrovski, T.~Schaul, D.~Saxton, and R.~Munos.
\newblock Unifying count-based exploration and intrinsic motivation.
\newblock {\em Advances in neural information processing systems}, 29, 2016.

\bibitem{berner2019dota}
C.~Berner, G.~Brockman, B.~Chan, V.~Cheung, P.~D{\k{e}}biak, C.~Dennison, D.~Farhi, Q.~Fischer, S.~Hashme, C.~Hesse, et~al.
\newblock Dota 2 with large scale deep reinforcement learning.
\newblock {\em arXiv preprint arXiv:1912.06680}, 2019.

\bibitem{choi2018contingency}
J.~Choi, Y.~Guo, M.~Moczulski, J.~Oh, N.~Wu, M.~Norouzi, and H.~Lee.
\newblock Contingency-aware exploration in reinforcement learning.
\newblock {\em arXiv preprint arXiv:1811.01483}, 2018.

\bibitem{clements2019estimating}
W.~R. Clements, B.~V. Delft, B.-M. Robaglia, R.~B. Slaoui, and S.~Toth.
\newblock Estimating risk and uncertainty in deep reinforcement learning.
\newblock {\em arXiv preprint arXiv:1905.09638}, 2019.

\bibitem{fu2020d4rl}
J.~Fu, A.~Kumar, O.~Nachum, G.~Tucker, and S.~Levine.
\newblock D4rl: Datasets for deep data-driven reinforcement learning.
\newblock {\em arXiv preprint arXiv:2004.07219}, 2020.

\bibitem{fujimoto2021minimalist}
S.~Fujimoto and S.~S. Gu.
\newblock A minimalist approach to offline reinforcement learning.
\newblock {\em Advances in neural information processing systems}, 34:20132--20145, 2021.

\bibitem{fujimoto2019off}
S.~Fujimoto, D.~Meger, and D.~Precup.
\newblock Off-policy deep reinforcement learning without exploration.
\newblock In {\em International conference on machine learning}, pages 2052--2062. PMLR, 2019.

\bibitem{gal2016improving}
Y.~Gal, R.~McAllister, and C.~E. Rasmussen.
\newblock Improving pilco with bayesian neural network dynamics models.
\newblock In {\em Data-efficient machine learning workshop, ICML}, volume~4, page~25, 2016.

\bibitem{haarnoja2018soft}
T.~Haarnoja, A.~Zhou, P.~Abbeel, and S.~Levine.
\newblock Soft actor-critic: Off-policy maximum entropy deep reinforcement learning with a stochastic actor.
\newblock In {\em International conference on machine learning}, pages 1861--1870. PMLR, 2018.

\bibitem{hoeffding1994probability}
W.~Hoeffding.
\newblock Probability inequalities for sums of bounded random variables.
\newblock {\em The collected works of Wassily Hoeffding}, pages 409--426, 1994.

\bibitem{hoel2020reinforcement}
C.-J. Hoel, T.~Tram, and J.~Sj{\"o}berg.
\newblock Reinforcement learning with uncertainty estimation for tactical decision-making in intersections.
\newblock In {\em 2020 IEEE 23rd international conference on intelligent transportation systems (ITSC)}, pages 1--7. IEEE, 2020.

\bibitem{hong2022confidence}
J.~Hong, A.~Kumar, and S.~Levine.
\newblock Confidence-conditioned value functions for offline reinforcement learning.
\newblock {\em arXiv preprint arXiv:2212.04607}, 2022.

\bibitem{jin2020provably}
C.~Jin, Z.~Yang, Z.~Wang, and M.~I. Jordan.
\newblock Provably efficient reinforcement learning with linear function approximation.
\newblock In {\em Conference on Learning Theory}, pages 2137--2143. PMLR, 2020.

\bibitem{kaufmann2023champion}
E.~Kaufmann, L.~Bauersfeld, A.~Loquercio, M.~M$\ddot{x}$ller, V.~Koltun, and D.~Scaramuzza.
\newblock Champion-level drone racing using deep reinforcement learning.
\newblock {\em Nature}, 620(7976):982--987, 2023.

\bibitem{kidambi2020morel}
R.~Kidambi, A.~Rajeswaran, P.~Netrapalli, and T.~Joachims.
\newblock Morel: Model-based offline reinforcement learning.
\newblock {\em Advances in neural information processing systems}, 33:21810--21823, 2020.

\bibitem{kim2023model}
B.~Kim and M.-h. Oh.
\newblock Model-based offline reinforcement learning with count-based conservatism.
\newblock In {\em International Conference on Machine Learning}, pages 16728--16746. PMLR, 2023.

\bibitem{Kostrikov2021OfflineRL}
I.~Kostrikov, A.~Nair, and S.~Levine.
\newblock Offline reinforcement learning with implicit q-learning.
\newblock {\em ArXiv}, abs/2110.06169, 2021.

\bibitem{kumar2019stabilizing}
A.~Kumar, J.~Fu, M.~Soh, G.~Tucker, and S.~Levine.
\newblock Stabilizing off-policy q-learning via bootstrapping error reduction.
\newblock {\em Advances in Neural Information Processing Systems}, 32, 2019.

\bibitem{kumar2020conservative}
A.~Kumar, A.~Zhou, G.~Tucker, and S.~Levine.
\newblock Conservative q-learning for offline reinforcement learning.
\newblock {\em Advances in Neural Information Processing Systems}, 33:1179--1191, 2020.

\bibitem{kurutach2018model}
T.~Kurutach, I.~Clavera, Y.~Duan, A.~Tamar, and P.~Abbeel.
\newblock Model-ensemble trust-region policy optimization.
\newblock In {\em International Conference on Learning Representations}, 2018.

\bibitem{lee2021sunrise}
K.~Lee, M.~Laskin, A.~Srinivas, and P.~Abbeel.
\newblock Sunrise: A simple unified framework for ensemble learning in deep reinforcement learning.
\newblock In {\em International Conference on Machine Learning}, pages 6131--6141. PMLR, 2021.

\bibitem{levine2020offline}
S.~Levine, A.~Kumar, G.~Tucker, and J.~Fu.
\newblock Offline reinforcement learning: Tutorial, review, and perspectives on open problems.
\newblock {\em arXiv preprint arXiv:2005.01643}, 2020.

\bibitem{ma2024mutual}
X.~Ma, B.~Kang, Z.~Xu, M.~Lin, and S.~Yan.
\newblock Mutual information regularized offline reinforcement learning.
\newblock {\em Advances in Neural Information Processing Systems}, 36, 2024.

\bibitem{mao2024supported}
Y.~Mao, H.~Zhang, C.~Chen, Y.~Xu, and X.~Ji.
\newblock Supported value regularization for offline reinforcement learning.
\newblock {\em Advances in Neural Information Processing Systems}, 36, 2024.

\bibitem{mnih2015human}
V.~Mnih, K.~Kavukcuoglu, D.~Silver, A.~A. Rusu, J.~Veness, M.~G. Bellemare, A.~Graves, M.~Riedmiller, A.~K. Fidjeland, G.~Ostrovski, S.~Petersen, C.~Beattie, A.~Sadik, I.~Antonoglou, H.~King, D.~Kumaran, D.~Wierstra, S.~Legg, and D.~Hassabis.
\newblock Human-level control through deep reinforcement learning.
\newblock {\em nature}, 518(7540):529--533, 2015.

\bibitem{openai2023gpt4}
OpenAI.
\newblock Gpt-4 technical report, 2023.

\bibitem{osband2018randomized}
I.~Osband, J.~Aslanides, and A.~Cassirer.
\newblock Randomized prior functions for deep reinforcement learning.
\newblock {\em Advances in Neural Information Processing Systems}, 31, 2018.

\bibitem{ostrovski2017count}
G.~Ostrovski, M.~G. Bellemare, A.~Oord, and R.~Munos.
\newblock Count-based exploration with neural density models.
\newblock In {\em International conference on machine learning}, pages 2721--2730. PMLR, 2017.

\bibitem{prudencio2023survey}
R.~F. Prudencio, M.~R. O.~A. Maximo, and E.~L. Colombini.
\newblock A survey on offline reinforcement learning: Taxonomy, review, and open problems.
\newblock {\em IEEE Transactions on Neural Networks and Learning Systems}, 2023.

\bibitem{rezaeifar2022offline}
S.~Rezaeifar, R.~Dadashi, N.~Vieillard, L.~Hussenot, O.~Bachem, O.~Pietquin, and M.~Geist.
\newblock Offline reinforcement learning as anti-exploration.
\newblock In {\em Proceedings of the AAAI Conference on Artificial Intelligence}, volume~36, pages 8106--8114, 2022.

\bibitem{schulman2015trust}
J.~Schulman, S.~Levine, P.~Abbeel, M.~Jordan, and P.~Moritz.
\newblock Trust region policy optimization.
\newblock In {\em International conference on machine learning}, pages 1889--1897. PMLR, 2015.

\bibitem{silver2018general}
D.~Silver, T.~Hubert, J.~Schrittwieser, I.~Antonoglou, M.~Lai, A.~Guez, M.~Lanctot, L.~Sifre, D.~Kumaran, T.~Graepel, et~al.
\newblock A general reinforcement learning algorithm that masters chess, shogi, and go through self-play.
\newblock {\em Science}, 362(6419):1140--1144, 2018.

\bibitem{sutton2018reinforcement}
R.~S. Sutton and A.~G. Barto.
\newblock {\em Reinforcement learning: An introduction}.
\newblock MIT press, 2018.

\bibitem{tang2017exploration}
H.~Tang, R.~Houthooft, D.~Foote, A.~Stooke, X.~Chen, Y.~Duan, J.~Schulman, F.~D. Turck, and P.~Abbeel.
\newblock \# exploration: A study of count-based exploration for deep reinforcement learning.
\newblock {\em Advances in neural information processing systems}, 30, 2017.

\bibitem{van2016conditional}
A.~Van~den Oord, N.~Kalchbrenner, L.~Espeholt, O.~Vinyals, A.~Graves, et~al.
\newblock Conditional image generation with pixelcnn decoders.
\newblock {\em Advances in neural information processing systems}, 29, 2016.

\bibitem{vinyals2019grandmaster}
O.~Vinyals, I.~Babuschkin, W.~M. Czarnecki, M.~Mathieu, A.~Dudzik, J.~Chung, D.~H. Choi, R.~Powell, T.~Ewalds, P.~Georgiev, J.~Oh, D.~Horgan, M.~Kroiss, I.~Danihelka, A.~Huang, L.~Sifre, T.~Cai, J.~P. Agapiou, M.~Jaderberg, A.~S. Vezhnevets, R.~Leblond, T.~Pohlen, V.~Dalibard, D.~Budden, Y.~Sulsky, J.~Molloy, T.~L. Paine, C.~Gulcehre, Z.~Wang, T.~Pfaff, Y.~Wu, R.~Ring, D.~Yogatama, D.~Wünsch, K.~McKinney, O.~Smith, T.~Schaul, T.~Lillicrap, K.~Kavukcuoglu, D.~Hassabis, C.~Apps, and D.~Silver.
\newblock Grandmaster level in starcraft ii using multi-agent reinforcement learning.
\newblock {\em Nature}, 575(7782):350--354, 2019.

\bibitem{wang2020reward}
R.~Wang, S.~S. Du, L.~Yang, and R.~R. Salakhutdinov.
\newblock On reward-free reinforcement learning with linear function approximation.
\newblock {\em Advances in neural information processing systems}, 33:17816--17826, 2020.

\bibitem{wu2021uncertainty}
Y.~Wu, S.~Zhai, N.~Srivastava, J.~Susskind, J.~Zhang, R.~Salakhutdinov, and H.~Goh.
\newblock Uncertainty weighted actor-critic for offline reinforcement learning.
\newblock {\em arXiv preprint arXiv:2105.08140}, 2021.

\bibitem{yu2021reinforcement}
C.~Yu, J.~Liu, and S.~Nemati.
\newblock Reinforcement learning in healthcare: A survey.
\newblock {\em ACM Computing Surveys (CSUR)}, 55(1):1--36, 2021.

\end{thebibliography}

\end{document}